\documentclass[sigconf]{acmart}

\AtBeginDocument{%
  }

\copyrightyear{2025} 
\acmYear{2025} 
\setcopyright{rightsretained}

\acmConference[WWW '25]{Proceedings of the ACM Web Conference 2025}{April 28-May 2, 2025}{Sydney, NSW, Australia}
\acmBooktitle{Proceedings of the ACM Web Conference 2025 (WWW '25), April 28-May 2, 2025, Sydney, NSW, Australia}
\acmDOI{10.1145/3696410.3714660}
\acmISBN{979-8-4007-1274-6/25/04}

\makeatletter
\gdef\@copyrightpermission{
  \begin{minipage}{0.8\columnwidth}
   \href{https://creativecommons.org/licenses/by/4.0/}{This work is licensed under a Creative Commons Attribution International 4.0 License.}
  \end{minipage}
  \vspace{5pt}
}
\makeatother

\usepackage[utf8]{inputenc} %
\usepackage[T1]{fontenc}    %
\usepackage{hyperref}       %
\usepackage{url}            %
\usepackage{booktabs}       %
\usepackage{amsfonts}       %
\usepackage{nicefrac}       %
\usepackage{microtype}      %
\usepackage{xcolor} 
\usepackage{balance}%
\usepackage{mathtools}
\usepackage{amsthm}

\usepackage{multirow}
\usepackage{adjustbox}

\usepackage{microtype}
\usepackage{graphicx}
\usepackage{subfigure}[labelformat=empty]
\usepackage{booktabs} %
\usepackage{array,multirow}
\usepackage{subcaption}
\usepackage{enumitem}

\usepackage{amsmath}
\usepackage{mathtools}
\usepackage{definitions}
\usepackage{enumitem}
\usepackage{multirow}
\usepackage{graphicx}
\usepackage{subcaption}
\usepackage{booktabs}
\usepackage{pifont}
\usepackage{wrapfig}
\usepackage[medium,compact]{titlesec}
\usepackage{ctable} %
\usepackage{xcolor,colortbl}
\usepackage{arydshln}

\newtheorem{theorem}{Theorem}[section]

\newtheorem{definition}[theorem]{Definition}

\newtheorem{lemma}[theorem]{Lemma}
\newtheorem{remark}[theorem]{Remark}

\usepackage{titlesec}
\titlespacing*{\section}{0pt}{*1}{*1}

\renewcommand{\paragraph}[1]{\noindent {\bf #1}}

\begin{document}

\title{What's in a Query: Polarity-Aware Distribution-Based Fair Ranking}

\author{Aparna Balagopalan$^*$}
\email{aparnab@mit.edu}
\affiliation{\institution{Massachusetts Institute of Technology}
\country{Cambridge, USA}
}

\author{Kai Wang$^*$}
\email{kwang692@gatech.edu}
\affiliation{\institution{Georgia Institute of Technology}
\country{Atlanta, USA}
}

\author{Olawale Salaudeen}
\email{olawale@mit.edu}
\affiliation{\institution{Massachusetts Institute of Technology}
\country{Cambridge, USA}
}

\author{Asia Biega}
\email{asia.biega@mpi-sp.org}
\affiliation{
\institution{Max Planck Institute \\for Security and Privacy} 
\country{Bochum, Germany}
}

\author{Marzyeh Ghassemi}
\email{mghassem@mit.edu}
\affiliation{\institution{Massachusetts Institute of Technology}
\country{Cambridge, USA}
}

\renewcommand{\shortauthors}{Aparna Balagopalan, Kai Wang, Olawale Salaudeen, Asia Biega, \& Marzyeh Ghassemi}

\begin{abstract}
Machine learning-driven rankings, where individuals (or items) are ranked in response to a query, mediate search exposure or \emph{attention} in a variety of safety-critical settings. Thus, it is important to ensure that such rankings are fair.  Under the goal of equal opportunity, attention allocated to an individual on a ranking interface should be proportional to their relevance across search queries. In this work, we examine amortized fair ranking -- where relevance and attention are cumulated over a sequence of user queries to make fair ranking more feasible in practice. Unlike prior methods that operate on expected amortized attention for each individual, we define new divergence-based measures for attention distribution-based fairness in ranking (DistFaiR), characterizing unfairness as the divergence between the distribution of attention and relevance corresponding to an individual over time. This allows us to propose new definitions of unfairness, which are more reliable at test time. Second, we prove that group fairness is upper-bounded by individual fairness under this definition for a useful class of divergence measures, and experimentally show that maximizing individual fairness through an integer linear programming-based optimization is often beneficial to group fairness.
Lastly, we find that prior research in amortized fair ranking ignores critical information about queries, potentially leading to a \emph{fairwashing} risk in practice by making rankings appear more fair than they actually are.

\end{abstract}


\maketitle

\begin{figure}[htb!]
\centering
\begin{minipage}[htb!]
{\linewidth}
\centering
\includegraphics[width=\linewidth,trim={0.1cm 5.5cm 0.1cm 5.5cm},clip]{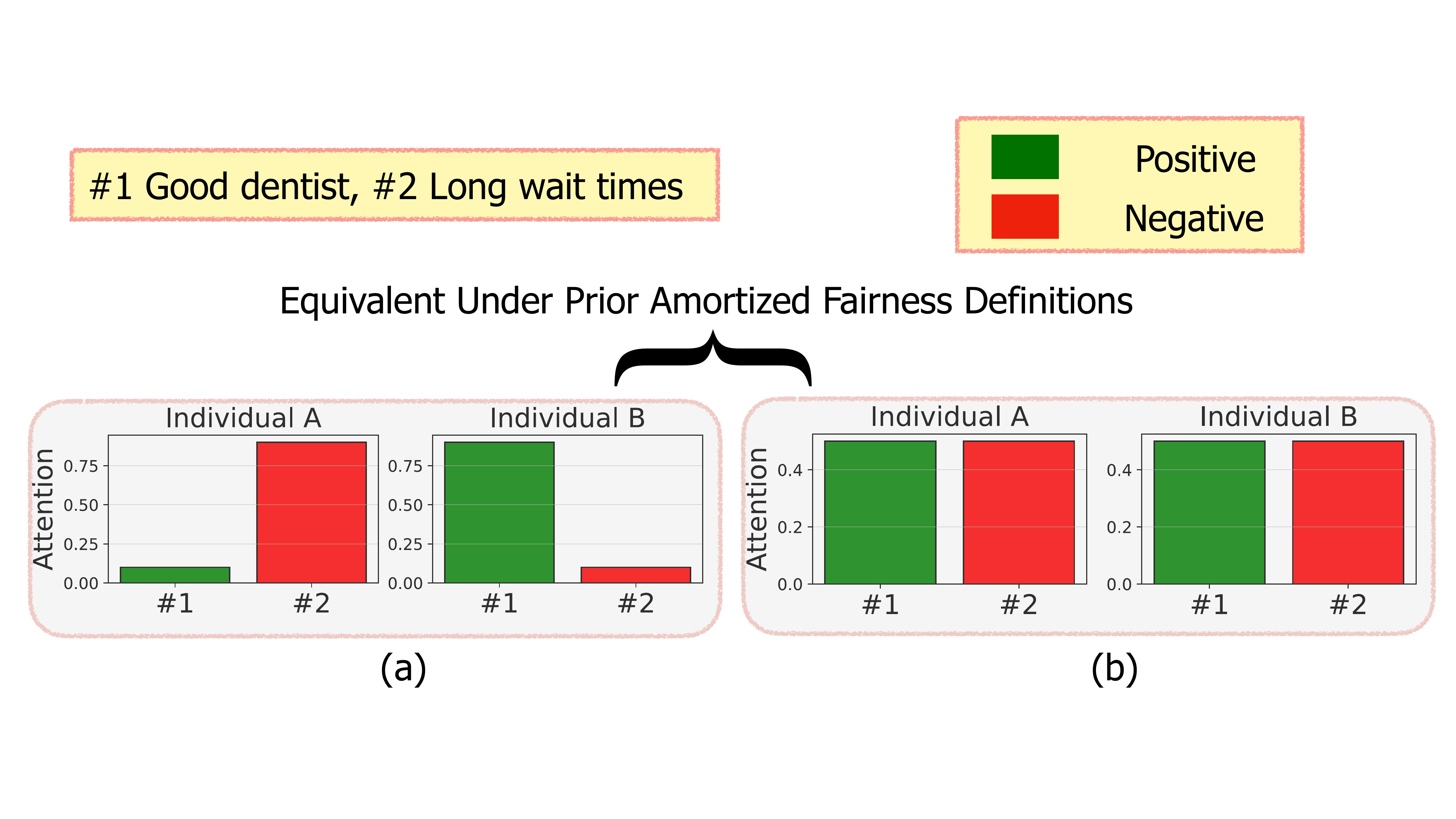}

\caption{\textbf{Past work in amortized fair ranking minimizes the differences between an individual's expected cumulative attention and relevance over queries, where queries are considered exchangeable.}  However, such formulations lack critical information about the distributions of attention and the properties of queries (e.g., polarity).  Our approach, DistFaiR, aims to overcome this.  The example here considers two search queries with opposite polarities. Both individuals are equally relevant, and have equal expected attention, but have different attention distributions in (a) and (b).}
\label{fig:need_for_distributions}
\end{minipage}
\end{figure}

\section{Introduction}
\label{submission}
Automated ranking systems are widely employed in several high-impact settings, such as ordering job candidates, guiding health-related decisions, and influencing purchasing decisions for safety-critical products~\cite{clarke2020overview,zeide2022silicon, zehlike2020reducing,hajian2016algorithmic}.
These systems directly influence access to critical resources, such as employment opportunities, healthcare, and safe consumer products, all of which significantly affect health and economic outcomes~\cite{saito2022fair,garcia2021maxmin,fabbri2020effect}, among others. However, prior work has shown that some automated rankings may be biased against some minoritized groups of individuals~\cite{biega2018equity,geyik2019fairness}: for example, women are less likely to occupy higher positions in rankings corresponding to searches made in some online hiring contexts~\cite{chen2018investigating}. The increased adoption of large language models (LLMs) as efficient text rankers~\cite{sanner2023large,zhuang2023open,hou2023large,gao2023chat} has the potential to increase the prevalence of automated rankings. However, this expanded usage risks amplifying existing biases in the distribution of user attention and economic opportunity. Mitigating such risks is a critical step towards building a responsible web-based system such as search engines whose performative power~\cite{hardt2022performative} to amplify bias has been demonstrated~\cite{mendler2024engine}.  \looseness=-1

One domain where automated ranking systems are ubiquitous is {\em search}~\cite{altman2005ranking}. Previous works have proposed several interventions and metrics~\cite{beutel2019fairness,naghiaei2022cpfair} to ensure that user attention in search is fairly distributed. In these frameworks, ranking algorithms are considered to be mediators of \emph{exposure} to searchers~\cite{joachims2021recommendations,schnabel2016recommendations}, where exposure is defined as the likelihood of \emph{visual attention} from searchers. A common intervention to achieve fair ranking is distributing rankings, and thereby attention, as a function of \emph{relevance}~\cite{singh2018fairness, biega2018equity}. 

Notably, fair exposure can be impossible in a single ranking where attention decays quicker than relevance, for instance, when all individuals have equal relevance and rankings have position bias w.r.t. attention (see Section~\ref{ref:section_ranking_def}). As a result, many proposed fairness interventions primarily focus on achieving fair exposure on the aggregate, i.e., over a sequence of queries (e.g., {``good dentist", ``good optometrists", ...})~\cite{biega2018equity,singh2018fairness}. That is, fairness is \emph{amortized} over a sequence, e.g., query $\#1$ and $\#2$ in Figure~\ref{fig:need_for_distributions}.

We identify two key limitations in current definitions of fair attention-based amortized ranking: (1) existing methods primarily focus on differences between the mean of attention and relevance distributions across queries, which fails to account for discrepancies in higher-order moments, such as variance or skewness, that may impact fairness (see Figure~\ref{fig:need_for_distributions} and Figure~\ref{fig:distribution_reliability} in Appendix for intuition), and (2) these methods assume that all attention is inherently positive, overlooking cases where unfairness arises due to disproportionate attention for negative or harmful queries compared to equally relevant counterparts, e.g., Figure~\ref{fig:fairwashing_amortized_ranking} in Appendix. 

Our approach, {\em distribution-aware fairness in ranking} (DistFaiR), overcomes these limitations. Our contributions are as follows:
\vspace{-1.05em}

\begin{itemize}
    \item We formalize a definition of amortized ranking fairness that accounts for differences (beyond means) in the distributions of cumulative relevance and cumulative exposure for individuals over a sequence of queries.
    \item We identify a set of measures that enable attention and relevance distribution-aware fairness in ranking (DistFaiR). We also consider a worst-case definition of fairness. %
    Specifically, we demonstrate theoretically and empirically that, for these measures, individual fairness upper bounds group fairness for the identified set of DistFaiR measures. Also, we show empirically that individual and group fairness are not always at odds, i.e., improving individual fairness often improves group fairness. 
    \item We demonstrate {\em fairwashing}, a phenomenon where a ranking appears to be more fair than it is when the polarities of queries are not accounted for. We propose polarity-dependent modifications to our newly proposed and existing fairness metrics that address this issue of fairwashing.
    
\end{itemize}

\vspace{-0.2em}

\section{Background and Related Work}
\vspace{-0.9em}
\paragraph{Fair Ranking Metrics and Interventions.} Rankings with high ranking quality may be unfair at a group or individual level~\cite{biega2018equity, diaz2020evaluating, singh2018fairness, morik2020controlling, bower2020individually,zehlike2021fairness,mehrotra2022fair}. Previous works have proposed various techniques to quantify and mitigate unfairness by allocating exposure or visual attention proportionally to relevance at a group or individual level~\cite{yang2023vertical,biega2018equity, diaz2020evaluating, singh2018fairness, morik2020controlling, bower2020individually,zehlike2021fairness,heuss2022fairness}. Some of these are in-processing (i.e., during ranking generation )~\cite{singh2019policy,bower2020individually,xu2024fairsync}, while some are post-processing~\cite{biega2018equity,sarvi2022understanding} interventions.  In some cases, relevance scores used to produce the ranking are jointly estimated along with fairness optimization~\cite{morik2020controlling,singh2019policy,bower2020individually}. Extensive work has also focused on proportion-based ranking, instead of exposure-based ranking~\cite{gorantla2023sampling,gorantla2021problem,geyik2019fairness}. We direct the interested reader to ~\cite{raj2022measuring} for a detailed review of fairness metrics. While prior papers have proposed some distribution-based measures~\cite{garcia2021maxmin,diaz2020evaluating}, these have been for fairness under stochastic ranking policies for a single query~\cite{gorantla2023sampling,singh2018fairness}. In contrast, we focus on the multi-query amortized setup and consider distributions driven by attention over a sequence of queries, rather than (only) due to the stochastic nature of rankings.

\looseness=-1
\paragraph{Trade-offs between Group and Individual Fairness.}
Prior work has proposed algorithms to optimize individual fairness without violating group ranking fairness or other constraints such as item diversity~\cite{garcia2021maxmin,saito2022fair,gorantla2021problem,gorantla2023sampling,flanigan2021fair}. Bower \emph{et al.} \citep{bower2020individually} showed empirically that improving individual fairness is beneficial to improving group fairness in in-processing fair ranking. To the best of our knowledge, no work has theoretically analyzed the relation between group and individual ranking fairness in the amortized setting (e.g., if one bounds the other). Note that similar analyses exist in the classification space~\cite{dwork2012fairness}. In this work, we concretely show that under the proposed definitions of fairness, group unfairness is upper-bounded by individual unfairness.

\paragraph{Impact of Queries in Ranking.}
To the best of our knowledge, no prior fairness metrics or interventions utilize information about the query itself in measuring fairness. Closest to our finding is recent work by Patro \emph{et al.}~\cite{patro2022fair}, where the authors observe that ``user attention may not directly translate to provider utility due to missing context-specific factors"~\cite{patro2022fair}. We expand on this observation and empirically show that attention-based metrics may fail specifically in a \emph{cross-query amortization} setup. Our finding is also a generalization of a recent finding~\cite{de2023unfair}, where it was shown that search results can be manipulated in an amortized setting. Our findings also broadly highlight the risk of fairwashing -- e.g., due to search engine manipulation -- when not considering query polarity. Lastly, while notions of multisided exposure fairness, group over-exposure, and under-exposure~\cite{wu2022joint,burke2017multisided,wang2021user} are also related to our problem (where the impact of queries or users are considered in fairness formulation), they still assume that all attention is positive. In our work, we propose a method to integrate real-valued query properties such as sentiment polarity into the fairness definition without making these assumptions.

\section{Amortized Fair Ranking}
\label{ref:section_ranking_def}
Given a query $q$ at time $t$ ($q_t$), we consider a fair ranking task where the goal is to order individuals corresponding to the query optimally~\cite{robertson1977probability}, given relevance score $r_i^t \in \mathbb{R}$ for each individual $i$. The task typically consists of three components: (i) the ``query", (ii) the set of individuals to be ranked, and (iii) the relevance scores. Due to position bias, individuals gain exposure based on their position in the ranking, which directly influences the attention they receive~\cite{biega2018equity,singh2018fairness}. Under the normative principle of equal opportunity, the objective of exposure-based fair ranking is to assign rankings such that the attention allocated to each individual is proportional to their merit~\cite{singh2019policy,zehlike2021fairness,zehlike2022fairness,balagopalan2023role}. In practical terms, merit is operationalized as a value proportional to relevance.

The concept of amortized fair ranking in existing literature seeks to find a sequence of ranking assignments that minimize the discrepancy between the average cumulative attention and the average cumulative relevance of individuals (or groups) over time. Said differently, relevance and attention are averaged over a sequence of queries, and the goal is to ensure fairness over this horizon. In this section, we introduce the notations, definitions, and limitations associated with amortized fair ranking for fair attention allocation. Furthermore, we introduce our distribution and polarity-aware generalization of amortized fair ranking, which results in a more robust solution to the normative goal of equal opportunity.

\subsection{Notation}
Consider a dataset $\mathcal{D}$ of $n$ individuals. Note that we use the term ``individual" here interchangeably with any item or entity being ranked throughout the paper. Each individual $i$ belongs to a group $g \in \Gcal$, where $\Gcal$ represents the set of $G$ possible groups. Let $g_k$ denote the $k^{th}$ group in $\Gcal$ where $k \le G$ and denote group membership as $i \in g_k$ where $g_k \subset \Dcal$. We assume that each individual belongs to exactly one group. Denote $q$ to be a sequence of queries, where queries $q_t$ are submitted at discrete time steps $t \in \Tcal$ and $\mathcal{T} = \{1, 2, \ldots, T\}$. A ranking system accepts each of these $T$ queries independently and returns a distinct ranking of $n$ individuals for each query $q_t \in \Qcal$, where $\Qcal$ denotes the space of all queries.

For each individual $i$ at time $t$, let the binary random variables $X^t_i$ and $Y^t_i$ denote the attention and relevance, respectively. Specifically, $X^t_i \sim \text{Bernoulli}(a^t_i)$~\cite{wang2018position} denotes whether individual $i$ receives attention at time $t$, and $Y^t_i \sim \text{Bernoulli}(r^t_i)$ denotes targets for the attention-distribution based on the relevance of individual $i$ to $q_t$, the query at time $t$. We assume that $X^t_i$ and $X^t_j$ are independent $\forall t$ when $i \neq j$. That is, under the attention models we study, the likelihood of attention is independent of other individuals being ranked, similar to prior work in fair ranking~\cite{biega2018equity,singh2018fairness}. We also assume that queries are independent. Crucially, for each time step $t$, the total attention and relevance are constrained such that \[\sum_{i \in n} a_i^t = 1 \quad \quad \text{and} \quad \quad \sum_{i \in n} r_i^t = 1,\]
such that attention and relevance for individuals are normalized~\cite{biega2018equity} with respect to $n$ individuals at each time step.

Furthermore, denote the cumulative attention and relevance distributions for individual $i$ over the full sequence of queries (all time steps) \[X_i = \sum_{t \in \mathcal{T}} X^t_i \quad \quad \text{and} \quad \quad Y_i = \sum_{t \in \mathcal{T}} Y^t_i,\]
respectively. A glossary of key terms is in Appendix Table~\ref{tab:glossary}.

\begin{theorem} \label{theorem:chernoff}
Let $X_i^t \sim \text{Bernoulli}(p_i^t)$ and
\[
    X_i = \sum_{t\in \Tcal} X_i^t.
\]
Then, for any $\delta > 0$, we have the following:
\[
    P\left(|X_i - \EE[X_i]| \geq \delta \EE[X_i]\right) \leq 2\exp\left(-\frac{\delta^2 \EE[X_i]}{2 + \delta}\right).
\]
\end{theorem}

\begin{remark}
    Note that Theorem \ref{theorem:chernoff}'s bound depends solely on the expected value of the cumulative attention (and relevance), not the number of queries observed.
\end{remark}

Theorem \ref{theorem:chernoff} bounds the likelihood of observing a deviation based on delta $\delta$ from the true cumulative attention for an individual over time $t$. We can apply the same exact bound for cumulative relevance.

In our setup, the ranking quality at each timestep $t$ is evaluated using the Discounted Cumulative Gain (DCG) at rank $K$, denoted as DCG@K~\cite{jarvelin2002cumulated}. The DCG@K score measures the quality of the top-$K$ ranked individuals based on their relevance, adjusting for the rank position using a logarithmic discount factor.

\subsection{Motivation For Amortized Fairness Across Different Queries}

This work focuses on a class of attention weights where user attention only depends on their ranking position. We assume that attention is proportional to position and follows a distribution informed by domain knowledge. For example, one such distribution used in several prior works is the log-decaying attention distribution~\cite{singh2018fairness,castillo2019fairness}. Under this distribution, at time $t$, if an individual $i$ is at position $j$, their attention score $a^t_i \propto \frac{1}{log(j+1)}$. 

Individual $i$'s attention $X^t_i$ is distributed as $A^t_i=\text{Bernoulli}(a^t_i)$, where $a^t_i$ is normalized.  Ideally, the probability of an individual receiving attention is proportional to their relevance score $r^t_i$, where relevance scores are $0-1$ normalized across all individuals for a given query.  Thus, under a fair ranking, individual $i$'s attention should be similarly distributed as their relevance, i.e., $a^t_i \approx r^t_i$. However, as mentioned above, the rate at which attention decays across positions in a ranking is usually very different from the variation in relevance across individuals. This makes it difficult to match the attention distribution to that of relevance within a single ranking.  Thus, it may be impossible to achieve the targeted fair attention distribution within a single deterministic ranking~\cite{biega2018equity,diaz2020evaluating}.

Alternatively, we compare \emph{cumulative} --- \emph{amortized}  --- attention and relevance over time. We also assume a more realistic multi-query setup since search systems typically process many queries over time. That is, we consider \emph{online ranking} where a sequence of queries (with corresponding relevance score per individual) arrive over time. We post-process the ranking corresponding to each query (without knowledge of the future queries) to improve fairness.

\subsection{Current Limitations}
Current amortized fairness metrics have two primary limitations: 

\emph{Insufficient measures of distributional differences between cumulated attention and relevance.} Current definitions compare \emph{expected (average) attention} ($\sum_{t \in \mathcal{T}} a_i^t$) to \emph{expected (average) relevance} ($\sum_{t \in \mathcal{T}} r_i^t$) across rankings, which leads to less reliably fair solutions for attention and relevance distributions where first moments (means) are not sufficient statistics (see Appendix Figure~\ref{fig:distribution_reliability}). %

\emph{Failure to capture the impact of query polarity:} Fairness definitions in the literature currently assume that all attention is good. However, increased attention in the context of queries with negative connotations relative to other similarly relevant individuals can lead to unfairness (see Appendix Figure~\ref{fig:fairwashing_amortized_ranking}). %
Hence, incorporating query polarity is necessary to model the real-world impact of unfair rankings. %

\subsection{Problem Statement}
We consider (un)fairness to be a function:
\begin{equation}
  f \colon \Pcal(\Xcal) \times \Pcal(\Ycal) \times \RR^T \longmapsto \mathbb{R}
\end{equation}
where $A \in \Pcal(\Xcal)$ denotes a distribution of cumulative attention and $R \in \Pcal(\Ycal)$ denotes a distribution of cumulative relevance for an individual. $Polarity$ is a vector containing the polarity of each of $T$ queries over which attention and relevance are cumulated. A lower value is desired.

Our task is to find a class of such functions such that:
\begin{equation}
    f(A, R, \text{polarity}) = 0 \implies \text{set of rankings is fair}
\end{equation}
We define $f$ to take the form of a scoring function for distribution-aware fairness in ranking (DistFaiR) and identify compatible measures for cumulative attention and relevance distributions in Section~\ref{sec:defining_amortized_distribution_fairness}. We then show that these measures can be modified to depend on query polarity in measurement (Section~\ref{sec:accounting_query_polarity}). Lastly, we test the sensitivity of current fair ranking metrics to polarity. This is an important step to assess \emph{fairwashing} effects in rankings~\cite{aivodji2019fairwashing}.  That is, whether fairness measured using query polarity is higher than that measured without using query polarity.

\newcommand{\cummA}{A_i}
\newcommand{\cummR}{R_i}
\newcommand{\groupcummA}{A_{g_{k}}}
\newcommand{\groupcummR}{R_{g_{k}}}

\section{Distribution-aware Fairness in Ranking (DistFaiR)}
\label{sec:defining_amortized_distribution_fairness}

We propose new distribution-based definitions of amortized fairness. We denote $\cummA$ and $\cummR$ to be the distribution of an individual $i$'s cumulative attention and relevance till time $T$ respectively. This is in contrast with prior definitions~\cite{biega2018equity,singh2018fairness,morik2020controlling}, where only the mean of the attention distribution over queries is considered for individuals and groups. We start by defining a class of amortized individual and group unfairness (DistFaiR) and then theoretically characterize a relationship between the two for a class of discrepancy measures.

\subsection{Defining Amortized Fairness}
\begin{definition}[DistFaiR-Divergence] \label{def:divergence}
Given two probability distributions $P$ and $Q$ over a common sample space $\Omega$, a divergence $D(P \| Q)$ is a function with the following properties:

\begin{enumerate}
    \item {\em Non-negativity}: $D(P \| Q) \geq 0$
    \item {\em Positivity}: $D(P \| Q) = 0$ if and only if $P = Q$
\end{enumerate}

\begin{lemma}
Define the following:
    \begin{align*}
        D_{L_1}(P \| Q) &= |\mu_P - \mu_Q|
    \end{align*}

    $D_{L_1}$ satisfies definition \ref{def:divergence} for $P$ and $Q$ when $\mu_P$ and $\mu_Q$ are sufficient statistics for their respective distributions. Additionally, it is subadditive, positively homogeneous, and scales under averages.
\end{lemma}

\end{definition}

\subsubsection{Individual Fairness}
\begin{definition}[Amortized Individual Unfairness] 
Amortized Individual Unfairness for a set of individuals is defined as the maximum distance between the distributions of cumulative relevance and cumulative attention over a sequence of queries up to time $T$. Specifically, the unfairness is given by:
\[
\text{Unfairness} = \max_{i \in \{1,2,\dots,n\}} D(\cummA, \cummR),
\]
where $i$ indexes the individuals to be ranked, and $D$ is a divergence.
\end{definition}

Notably, this definition differs from past definitions of amortized fairness~\cite{biega2018equity,singh2018fairness,singh2019policy,raj2022measuring} as follows: (1) the distribution-based fairness definition allows for distributions attention and relevance that are not fully specified by their means, (2) considers a worst-case notion of individual unfairness. For example, in ~\citep{biega2018equity}, unfairness is defined to be the $L_1$ distance of difference between cumulative relevance and cumulative exposure scores allocated to a set of $n$ individuals over $T$ queries. In our framework, this is equivalent to choosing a metric $d(P, Q) = |\EE[P] - \EE[Q]]|$, or the absolute difference in expectations of the two distributions. However, this only captures discrepancies between distributions $P, Q$ where means are sufficient statistics, e.g., Guassians with fixed variances or Exponential with rate parameters reciprocal to the mean. Appendix \ref{sec:dist_example} demonstrates that divergences, which capture properties of distributional difference beyond means, give a more robust and realistic definition of unfairness. %

\subsubsection{Group Fairness}~\\
We extend the previous definition to group level by defining the relevance and attention of a group as the average relevance and attention of individuals belonging to that group, respectively. The attention and relevance of a group $g_k \subset [n]$ at time $t$ respectively are random variables:
\begin{align}
    X_{g_{k}}^{t} = \frac{1}{|g_{k}|}\sum\nolimits_{i \in g_{k}} X_i^t & \quad \text{ and } \quad Y_{g_{k}}^{t} = \frac{1}{|g_{k}|} \sum\nolimits_{i \in g_k} Y_i^t ,
\end{align}
where $|g_{k}|$ denotes the number of individuals in group $g_{k}$, with $|g_{k}| \geq 1$. We can also apply Theorem \ref{theorem:chernoff} to quantify the tail probability of group level relevance and attention.

The relevance distribution and attention in a group $g_k \subset [n]$ throughout time $t \in \Tcal$ are respectively:
    \begin{align}
        X_{g_{k}} = \sum\nolimits_{t \in \Tcal} X_{g_{k}}^{t} & \quad \text{ and } \quad Y_{g_k} = \sum\nolimits_{t \in \Tcal} Y_{g_{k}}^{t} .
    \end{align}
    
Denote $\groupcummA$ and $\groupcummR$ as the distributions of cumulative attention and relevance from which $X_{g_k}$ and $Y_{g_k}$ are generated.

\begin{definition}[Amortized Group Unfairness]\label{def:group-unfairness}
Amortized Group Unfairness for a set of $G$ groups is defined as the maximum distance between the distributions of cumulative relevance and cumulative attention scores across a sequence of queries up to time $T$ for each group. Each individual is assumed to belong to exactly one of the $G$ groups. Formally, group unfairness is expressed as:
\[
\text{Group Unfairness} = \max_{g_k \in \Gcal} D(\groupcummA \| \groupcummR),
\]
where $D$ represents a divergence, $g_k$ denotes the $k$-th group, and $\groupcummA$ and $\groupcummR$ represent the distributions of cumulative attention and cumulative relevance for group $g_k$, respectively.
\end{definition}

We refer our definitions of amortized individual and group unfairness above as \textit{DistFaiR}.

\subsection{Individual Fairness v.s. Group Fairness}

\begin{theorem}\label{theorem:indiv_group}
    For any jointly convex DistFaiR divergence that is subadditive under the convolution operation, positively homogeneous with degree $s$, and scales under averages, amortized group fairness is upper-bounded by amortized individual fairness. Specifically, we have the following inequality:
    \begin{align}
        \max_{g_k \in \Gcal}D(\groupcummA \| \groupcummR) \leq  \max_{i \in \Dcal} D(A_i \| R_i) \hspace{0.5cm} 
        \forall g_k \in \mathcal{G}
    \end{align}
\end{theorem}

Proof provided in Appendix \ref{ref:sec_proof}.

Theorem \ref{theorem:indiv_group} shows that improving individual fairness does not adversely affect group fairness for a class of divergence measures --- optimizing for individual fairness may improve group fairness. Thus, while individual fairness is good criteria, it may not always be possible to ensure individual fairness for some divergence measures (e.g., due to computational infeasibility). In such cases, group fairness constraints could be considered weaker versions of individual fairness, and could be be used more broadly, depending on the underlying normative goals.

\subsection{Amortized Fairness Re-ranking with Quality Constraints}
\label{sec:re-ranking}
Theorem \ref{theorem:indiv_group} motivates optimizing for individual (un)fairness. Accordingly, we design an objective function corresponding to individual unfairness to be minimized, similar to Biega \emph{et al.} \cite{biega2018equity}.
\begin{align}
     \min\nolimits_{M_{i,j}^t} & \quad \text{max}_{i \in n} \quad D(\cummA \| \cummR) \quad  \text{(individual fairness)} \\
    \text{s.t.} & \quad  \sum\nolimits_{j=1}^k \sum\nolimits_{i=1}^n \frac{r_i^t}{\log_2(j+1)} M^t_{i,j} \geq \theta*\rho(t) \quad \text{for each } t \in \mathcal{T} \\
    & \quad M^t_{i,j} \in \{ 0, 1 \} \quad \forall i,j\\
    & \quad \sum\nolimits_{i} M^t_{i,j} = 1 \quad \forall j \\
    & \quad \sum\nolimits_{j} M^t_{i,j} = 1 \quad \forall i
\end{align}
where $A_i$ and $R_i$ denote cumulative attention and relevance for individual $i$ till time $t$, $M^{t}_{i,j}$ is a binary variable indicating if individual $i$ is present at rank $j$ for the query at time $t$. $\rho(t)$ indicates the DCG (quality) of the ranking at time $t$. Constraint (7) ensures that the quality of the updated ranking does not decrease beyond a given threshold $\theta$. Additionally, constraints (8) and (9) ensure that each individual can be ranked only once in a ranking and no positions are empty, respectively. Given the large size of the variable space, when $n$ is large, we pre-filter the rankings and set $M_{i,j}^t$ to be fixed when $j > K$ for some known $K \in \mathbb{N} < n$. Thus, we only re-order the top-$K$ within each ranking. 

\paragraph{Integer Linear Programming Formulation}
We solve the above optimization problem using integer linear programming and/or integer quadratic programming. We rely on an open-source toolkit, Gurobi~\cite{achterberg2019s} to perform all optimizations where minimizing our objective yields amortized fairness. We study \emph{online optimization} where a new query arrives at each $t$, and hence $M_{i,j}^t$ is optimized at each time step, with knowledge of prior assignments~\cite{biega2018equity}, but no knowledge of the future. Further details can be found in Appendix~\ref{sec:ilp_appendix}. 

\begin{table}
\centering
\caption{Summary statistics of all datasets. The relevance score in the \texttt{rateMDs} dataset and the query utility score in the \texttt{FairTREC2021} dataset are generated using pre-trained LLMs.}
\label{tab:ds_summary}
\footnotesize{
\begin{minipage}{\linewidth}
\resizebox{\textwidth}{!}{

  \begin{tabular}{lcccccc}
  \toprule
    Dataset & \#Individuals &  \#Queries  & \#Groups & Relevance & Polarity\\
    \midrule
    \texttt{synth-binary} & 200 & 16 &2 & $\{0.99,1.01\}$ & $\{-1,1\}$ \\
    \texttt{synth-cont} & 200 & 16 & 2 & Cont. & $\{-1,1\}$\\
    \texttt{rateMDs} & 6.2k & 60 & 2 & Cont. & $\{-1,1\}$\\
    \texttt{FairTREC 2021} & 13.5k & 49 & 3 & $\{0,1\}$ & $[-1,1]$\\
    \bottomrule
    \end{tabular}
}
\end{minipage}
}
\end{table}

\section{Accounting for Query Polarity}
\newcommand{\polarity}{\eta}
\label{sec:accounting_query_polarity}
Prior work in fair ranking assumes that all attention is positive~\cite{biega2018equity,singh2018fairness} and query independent, implying that achieving a higher rank is universally desirable. However, individuals should not be given higher attention for queries with negative connotations than those with similar relevance. Consequently, we extend our fairness definition to account for query properties such as \emph{sentiment polarity} by introducing a {\em context function} associated with each query. 

In this work, we focus on the scalar sentiment polarity associated with each query. Alternative properties may include the clarity of the query, the perceived economic value associated with being highly ranked for the query, etc. This variable will be influenced --- at least partially --- by the information contained in the query, and may be positive or negative, determining if a higher or lower ranking is more favorable. We also show how this context function can be extended beyond scalars to include a vector of query properties.

Let $\Tilde{X}_i^t$ represent a random variable denoting attention allocated to individual $i$ at time t that incorporates query polarity. Assuming that polarity is searcher-independent (no personalization), it can be decomposed into: (1) the real-world value associated with the attention allocated in response to a query at time $t$ and (2) individual attention. Similar to previous works on fair ranking, we may assume that searcher attention can be modeled well with models like position bias~\cite{chuklin2015click,craswell2008experimental}. %

We denote the context function $\polarity(q_t)$ as the polarity associated with query at time $t$. Query polarity-aware attention and relevance allocated to individual $i$ at time $t$ is then defined as:
\begin{equation}
    \Tilde{X}_i^t = X_i^t \cdot \polarity(q_t) \quad \quad \text{ and } \quad \quad \Tilde{Y}_i^t = Y_i^t \times \polarity(q_t),
\end{equation}
where each corresponds to a cumulative distribution $\Tilde{A}_i$ and $\Tilde{R}_i$, respectively. This formulation is free of two assumptions inherent to the exposure-based fairness metrics: (1) the contribution of exposure to amortized ranking is now dependent on properties of the query and (2) exposure can be any real-valued number. Notably, $\eta(q_t) \in \RR$, including \emph{negative} values and \emph{zero}, unlike previous work. Then, amortized fairness under DistFaiR can be computed over time, with all notations following from the previous section. %
We refer to fairness measures defined in the prior section as \emph{query polarity-agnostic}, and those relying on $\eta(q_t)$ as \emph{query polarity-aware}.

\begin{theorem} \label{theorem:hoeffdings}
Let $X_i^t \sim \text{Bernoulli}(p_i^t)$ and $\eta(q_t) \in [a_t, b_t]$; $a_t, b_t \in \RR$. With a slight abuse of notation, let $\Tilde{X}_i^t = X_i^t\cdot \eta(q_t) \in [a_t, b_t]$ and
\[
    \Tilde{X}_i = \sum_{t\in \Tcal} \Tilde{X}_i^t,
\]
Then, for any $\delta > 0$, we have the following:
\[
    P\left(|\Tilde{X}_i - \EE[\Tilde{X}_i]| \geq \delta\right) \leq 2\exp\left(-\frac{2\delta^2}{\sum_{t \in \Tcal} (b_t - a_t)^2}\right).
\]
\end{theorem}

\begin{remark}
    Unlike in Theorem \ref{theorem:chernoff}, the bounds in Theorem \ref{theorem:hoeffdings} now depend on both the expected value of the cumulative attention (and relevance) and the number of queries observed.
\end{remark}

Thereom \ref{theorem:chernoff} bounds the likelihood of observing a given deviation $\delta$ from the true polarity-aware cumulative attention for an individual over time $t$. We can apply the same exact bound for polarity-aware cumulative relevance.

\begin{table*}[
htb!]
\centering
\caption{{Individual fairness improves with DistFaiR re-ranking intervention, but the difference depends on the divergence measure used.} We show \emph{relative improvement} in fairness post- fair ranking intervention with respect to the original ranking. The columns (i.e., $\Delta$ measure) correspond to different fairness measures, while each row corresponds to a fair re-ranking method. Post-processing the rankings with DistFaiR improves individual fairness across datasets. Group fairness also improves with DistFaiR in most cases. Arrows indicate direction of better performance, with best performance bolded for each fairness metric. {\em Note that the criterion of the fairness scores varies across cross-columns, so cross-column comparisons are incorrect. }}
\label{tab:full_results_main}
\adjustbox{max width=0.8\linewidth}{%
 \begin{tabular}{llcccccc}
\toprule
Dataset & Method  & \multicolumn{3}{c}{Relative Change in Individual Fairness ($\uparrow$)}  &\multicolumn{3}{c}{Relative Change in Group Fairness ($\uparrow$)} \\
\cmidrule(r){3-5}\cmidrule(r){6-8}\\
 &  & $\Delta$ DistFaiR ($L_1$) & $\Delta$ DistFaiR ($L_2^{var}$) & $\Delta$ DistFaiR ($W_1$)  & $\Delta$ DistFaiR ($L_1$) & $\Delta$ DistFaiR ($L_2^{var}$) & $\Delta$ DistFaiR ($W_1$)  \\

\midrule
\multirow{5}{*}{\texttt{synth-binary}} & IAA & \textbf{82.50}\% & \textbf{90.89\%} & \textbf{68.18\%} & 8.84\% & 12.06\% & 0.54\%\\
 & FoE & 9.17\% & 14.65\% & 7.58\% & 20.81\% & 18.89\% & 1.85\%\\
  \cdashline{2-8}[1pt/3pt]
 & DistFaiR($L_1$) & \textbf{82.50\%} & \textbf{90.89\%} & \textbf{68.18\%} & 47.05\% & 40.87\% & 5.38\%\\
 & DistFaiR($L_2^{var})$ & 76.50\% & \textbf{90.89\%} & 65.02\% & 59.05\% & 48.55\% & \textbf{6.58}\% \\
 & DistFaiR($W_1$) & 77.81\% & 90.68\% & \textbf{68.18\%} & \textbf{76.10\%} & \textbf{67.26\%} & 3.69\%\\
 
\midrule

\multirow{5}{*}{\texttt{synth-cont}} & IAA & 61.39\% & \textbf{63.56\%} & 40.76\% & \textbf{38.52\%} & \textbf{42.22\%} & 32.32\% \\
& FoE & 2.02\% & -1.23\% & 3.74\% & -139.41\% & -273.17\% & 5.92\%\\
\cdashline{2-8}[1pt/3pt]
 & 
 DistFaiR($L_1$) & \textbf{62.02\%} & 60.42\% & 39.64\% & -36.10\% & -60.62\% & 66.89\%\\
 &
 DistFaiR($L_2^{var}$) & 61.84\% & 62.20\% & 39.95\% & -39.66\% & -75.31\%  & \textbf{67.88\%}\\
 & DistFaiR($W_1$) & 51.22\% & 58.29\% & \textbf{40.89\%} & -125.17\% & -235.10\% & 58.11\% \\
\midrule

\multirow{5}{*}{\texttt{FairTREC2021}}  & IAA &  \textbf{68.76\%} & \textbf{78.69\%} & \textbf{64.95\%} & 48.99\% & 76.02\% & \textbf{4.32\%} \\
& FoE & 17.52\% & 25.64\% & 17.36\% & 45.67\% & 68.15\% & -5.47\%\\
  \cdashline{2-8}[1pt/3pt]
 &DistFaiR($L_1$)& \textbf{68.76\%} & 77.75\% & \textbf{64.95\%} & \textbf{50.78\%} & \textbf{80.61\%} & 0.34\%\\
 & DistFaiR($L_2^{var}$) & \textbf{68.76\%} & \textbf{78.69\%} & 64.57\% & 42.62\% & 71.35\% & 0.36\%\\
 & DistFaiR($W_1$) &  68.27\% & 78.40\% & 64.77\% & 50.13\% & 78.46\% & -18.66\%\\

\midrule
\multirow{5}{*}{\texttt{rateMDs}}  & IAA &28.30\% & 46.03\% & 22.48\% & -7.94\% & -87.81\% & 5.82\%\\
& FoE & 3.37\% & 6.15\% & 2.99\% & 26.96\% & 40.71\% & 2.37\%\\
  \cdashline{2-8}[1pt/3pt]
 & 
 DistFaiR($L_1$)&\textbf{69.80\%} & 86.74\% & 62.90\% & \textbf{62.19\%} & \textbf{79.90\%} & 5.18\% \\
 &DistFaiR($L_2^{var}$) & 66.83\% & \textbf{86.75\%} & 59.74\% & 39.99\% & 60.80\% & 1.67\% \\
 & DistFaiR($W_1$) & 67.76\% & 85.34\% & \textbf{64.28\%} & 57.41\% & 78.10\% & \textbf{7.89\%} \\

\bottomrule
\end{tabular}}
\end{table*}

\section{Experiments: Online Fair Ranking}  

Our experiments are focused on an \emph{online fair ranking setup}, similar to ~\cite{biega2018equity}. We assume a realistic setup where a new query arrives at each time $t$, and we re-rank the system-produced ranking at time $t$ to improve fairness. We assume knowledge of attention allocated to individuals in rankings till time $t$ to produce this new fair ranking (i.e., a running memory of cumulative attention per individual)\footnote{Code: \url{https://github.com/MLforHealth/DistFaiR}}.

\subsection{Experimental Setup}
\paragraph{Datasets} We utilize two synthetic datasets which represent the setting described in the example shown in Figure~\ref{fig:fairwashing_amortized_ranking} where female individuals are allocated attention in four out of eight rankings (all with negative polarity) and two real-world fair ranking datasets~\cite{thawani2019online,trec-fair-ranking-2021}. A dataset summary is in Table~\ref{tab:ds_summary} and further details are provided in Appendix~\ref{sec:datasets}. We also benchmark the impact of query polarity on the Xing dataset~\cite{zehlike2017fa} in the Appendix (see Appendix~\ref{app:xing_dataset}).  Our empirical study focuses on post-processing fairness interventions, where individual relevance -- or ``groundtruth" -- scores are known~\cite{gorantla2023sampling}.

\paragraph{Query Properties} We experiment with polarity as the query property. The polarity score is synthetically generated for \texttt{synth-binary} and \texttt{synth-cont} and manually annotated for \texttt{rateMDs}. For the \texttt{FairTREC 2021} dataset, a pre-trained sentiment classification model is used to generate polarity~\cite{barbieri-etal-2020-tweeteval} (see Appendix~\ref{sec:datasets}). 

\paragraph{Distance Functions}
\label{sec:distance_metrics}
We consider three (pseudo) divergences metrics for measuring unfairness under DistFaiR:
\begin{itemize}[leftmargin=0.2in]
    \setlength\itemsep{0em}
    \item $\mathbf{L_1}$ distance is defined as the difference between the mean of two distributions: $D_{L_1}(A \| R) = |\mathbb{E}_{X \sim A}[X] - \mathbb{E}_{Y \sim R}[Y]|$. 
    \begin{itemize}
        \item This distance function has been studied in ~\cite{biega2018equity}, where fairness is computed as the sum of distance values across individuals and is referred to as the inequity of amortized attention (IAA).  We note that this function is generally not a proper divergence. However, for distributions $A$ and $R$ whose first moments are sufficient statistics, $D_{L_1}$ satisfies definition \ref{def:divergence}.
    \end{itemize}
    \item $\mathbf{L_2^{\text{var}}}$ distance is defined as the difference in mean and variance of two distributions\footnote{We use squared differences as we expect a square root of this to perform similarly.}: 
    \begin{align*} D_{L_2^{\text{var}}}(A \| R) &= (\mathbb{E}_{X \sim A}[X] - \mathbb{E}_{Y \sim R}[Y])^2 \\&+ (\sigma{}_{X \sim A}[X] - \sigma_{Y \sim R}[Y])^2.\end{align*}
    We note that $D_{L_2^{\text{var}}}$ benefits from $W_2$, a proper divergence, for two Gaussians, which has the properties for Theorem \ref{theorem:indiv_group}. %
    \item $\mathbf{W_{1}}$ distance is defined as the Wasserstein distance between the distribution of expected attention ($\{a_i^t\}_{t=1}^{\mathcal{T}}$) and distribution of expected relevance ($\{r_i^t\}_{t=1}^{\mathcal{T}}$) for an individual. $D_{W_1}(A \| R)=\frac{1}{T}\sum_{k=1}^{T} |a_i^{(k)} - r_i^{(k)}|$, where $(k)$ denotes the $k$th order statistic of empirical measures $\hat{A}_i$ and $\hat{R}_i$ from which $a_i^t$ and $r_i^t$ is sampled. 
    
\end{itemize}
\subsection{Evaluation}
\label{sec:metrics}
We utilized the following fairness criteria.

\paragraph{Individual Unfairness:} We use three different distance measures defined in Section~\ref{sec:distance_metrics} to measure the unfairness as:  DistFaiR($L_1$), DistFaiR($L_2^{\text{var}}$), and DistFaiR($W_1$).
The amortized fairness defined by DistFaiR($L_1$) is similar to the IAA fairness measure studied by \cite{biega2018equity}. However, we consider the \emph{worst-case} distance between attention and relevance distributions, while \cite{biega2018equity} consider the sum of difference across all individuals, which may hide heightened unfairness in some individuals. Our work also generalizes amortized fairness to include appropriate measurements of discrepancies between distributions that require higher-order moments to be specified, i.e., with $L_2^{\text{var}}$ and $W_1$ distances.

\paragraph{Group Unfairness:}
In addition to the group unfairness metrics directly induced by the three distance metrics using Definition~\ref{def:group-unfairness}, we consider a standard exposure-based group unfairness definitions: Exposed Utility Ratio (EUR). \cite{singh2018fairness,morik2020controlling} define the EUR difference as the absolute difference in the ratios of average exposure and average relevance between groups. We also measure an attention parity metric: Demographic Parity\cite{morik2020controlling} (DP).

\paragraph{Performance}
\label{sec:perf}
We measure the ranking quality via the DCG@K score, which is the sum of the relevance of the top-K individuals, with a logarithmic discount based on their position: 
$$\sum_{k=1}^{K}\frac{r_{\text{rank}(k)}^t}{log_2(k+1)},$$
where ${\text{rank}(k)}$ returns the index of the individual at rank $k$. After re-ranking, the DCG@K is normalized by the DCG@K of the previous (ideal) ranking to produce a normalized DCG@K between 0 and 1.

\subsection{Baselines: Fair Re-ranking }

\textbf{IAA}: A method to reduce inequity of amortized attention (IAA) introduced by Biega \emph{et al.}\cite{biega2018equity}. An ILP is solved to reduce the absolute difference in the mean of the cumulative attention and cumulative relevance distributions, summed across all individuals. In contrast, our method focuses on \emph{worst-case} minimization.

\textbf{FoE}: A linear program for ranking assignments with Birkhoff Von Neumann decomposition~\cite{lewandowski1986algorithmic} is solved to ensure fairness of exposure (FoE)~\cite{singh2018fairness}. The quality of rankings is maximized, with the constraint that the cumulative attention to relevance ratio is the same for all individuals. We re-rank only top-k individuals in each ranking. The original ranking is returned if solution is infeasible.\looseness=-1

\textbf{FIGR}~\cite{gorantla2021problem}: This method jointly aims to reduce ``underranking" (which is closely related to individual fairness) in rankings that are post-processed with group fairness constraints. Unlike the other baselines, this is a proportion-based re-ranker for each ranking, and does not explicitly consider attention distributions. Thus, we present results for this baseline in the Appendix. 

\vspace{-0.5em}
\subsection{Hyperparameter Tuning}
We stratified all datasets into two subsets: 50\% tuning and 50\% test sets, so no individuals or queries are present in both splits. All parameters (e.g., $\theta$; when tuned) are tuned using the tuning split. For \texttt{FairTREC 2021}, we use the full evaluation split, and do not perform any additional tuning -- we sample queries with replacement thrice to obtain variance. We run all optimization algorithms on a 3.2 GHz CPU with 16 GB RAM for $\leq$ 60 minutes, with a feasibility tolerance of $1e-9$. We set K=10 while measuring ranking quality and assume logarithmic discounts in attention till K=10 and zero otherwise.

We also pre-filter~\cite{biega2018equity}, and only re-rank the top-k individuals in each ranking. For moment-based divergences, $L_1$ and $L_2^{var}$,  we minimize maximum divergence only among the top-k at each step, as we found that this performs better. This means that even when the maximum divergence measure across all individuals cannot be reduced, we still re-rank to reduce the next possible highest divergence value. For $W_1$, we minimize divergence across all individuals. For the FoE baseline, constraints are set only for individuals in the top-k positions to make re-ranking feasible. Note that our results are sensitive to these pre-filtering choices. Post-tuning, we find that $k=50$ works well across datasets.

Our experimental flow is as follows: first, we implement our fair ranking definitions (DistFaiR) and compare to baselines. Second, we test if fairness metrics are affected by query polarity. Third, we perform several ablations for, e.g., the fairwashing effect.

\begin{figure*}[ht!]
\centering
\begin{subfigure}
    \centering
    \includegraphics[width=0.31\textwidth,trim={0 0.25cm 0 0},clip]{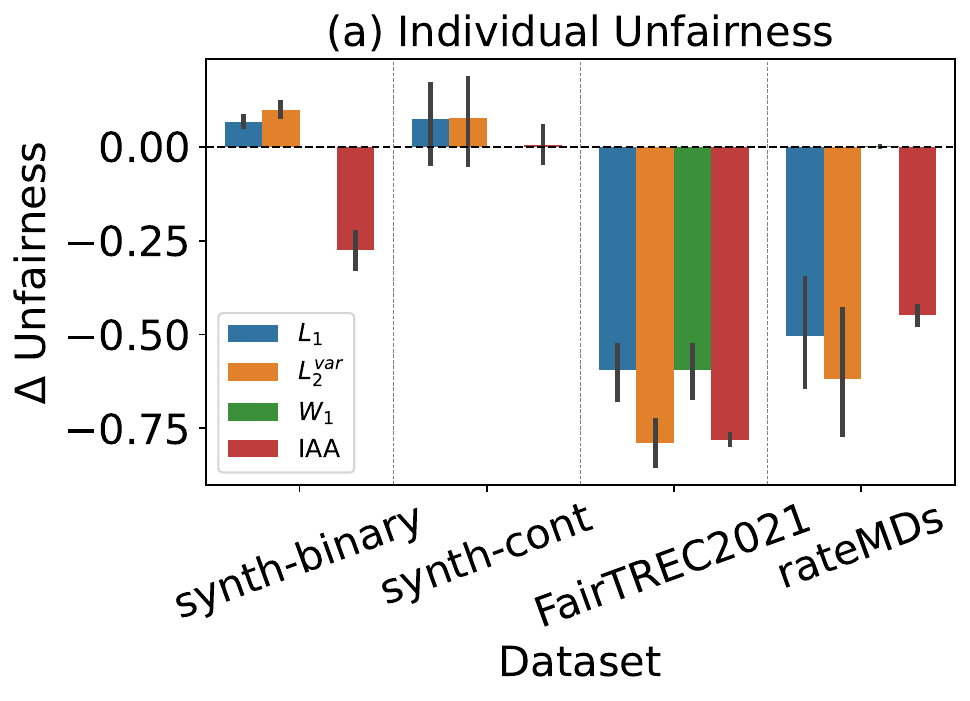}
\end{subfigure}%
\begin{subfigure}
    \centering
    \includegraphics[width=0.31\textwidth,trim={0 0.25cm 0 0},clip]{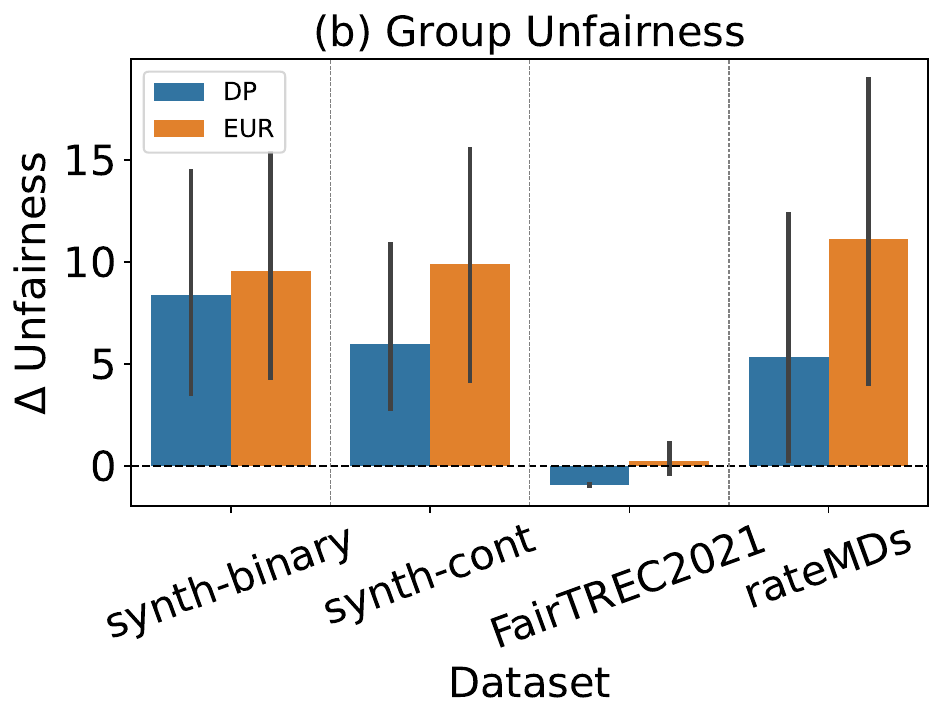}
\end{subfigure}%
\begin{subfigure}
    \centering
    \includegraphics[width=0.31\textwidth,trim={0 0.25cm 0 0},clip]{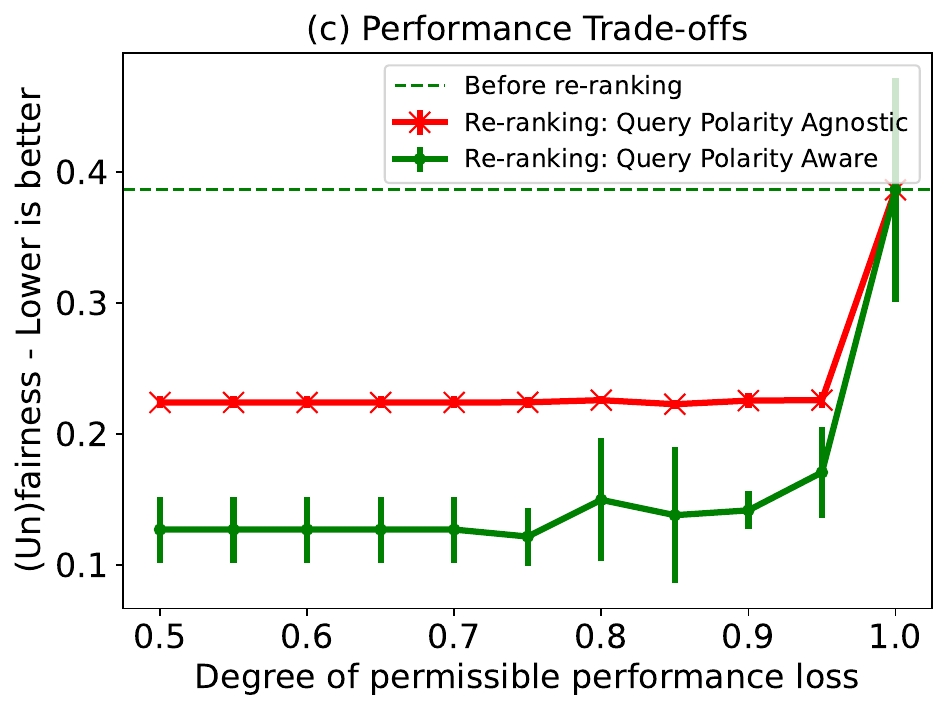}
\end{subfigure}
\caption{(a) and (b) show the difference (as relative change) between fairness metrics measured with and without query polarity. Query polarity impacts all amortized fairness metrics, as they differ from zero as seen in the plots. (rightmost) We plot the re-ranking performance of polarity agnostic and aware re-rankings under different permissible performance loss changes for the \texttt{synth-cont} dataset (DistFaiR($L_1$)), where we can see polarity agnostic re-ranking underperforms polarity aware re-ranking. %
}\label{fig:varying_theta}

\end{figure*}
\vspace{-0.5em}
\section{Results}

We measure the percentage change in unfairness pre- and post- re-ranking. A positive change -- decrease in unfairness -- is desired.

\paragraph{DistFaiR Improves Worst-Case  Fairness}
Table ~\ref{tab:full_results_main} shows that our re-rankings reduce individual unfairness, when unfairness is measured as the worst-case divergence measure between the attention and relevance distributions across individuals. We find that DistFaiR outperforms or performs on par with IAA. FIGR (Table~\ref{tab:figr})---which solves a different notion of ``underranking"---does not improve performance as measured by our metrics. Further, as expected, optimizing the divergence measure itself often leads to highest decrease in unfairness (for example, DistFaiR($W_1$) has highest improvement in fairness for the $\Delta$DistFaiR($W_1$) individual fairness measurement). Note that not all differences were statistically significant.

Additionally, as seen in Appendix Table~\ref{tab:group-re-ranking},  DistFaiR underperforms IAA-based re-ranking on the IAA metric. This makes sense because DistFaiR focuses on reducing worst-case divergence, while IAA focuses on the average across individuals. Thus, there appear to be tradeoffs between average and worst-case performance. Such observations have also been made in other fairness contexts~\cite{yang2023change}. %

\paragraph{Divergence Metric is an Important Design Choice.}
Our results show that the divergence metric is an important design choice. We find that the performance of $L_1$ and $L_2^{\text{var}}$ are close  (e.g., on \texttt{FairTREC2021}). We hypothesize that the  optimization with $W_1$ is more difficult, due to which performance improvements are smaller. Note that $L_2^{\text{var}}$ is the $W_2$ solution under assumptions of gaussianity. It is possible that using the $L_2^{\text{var}}$ measure could be an easier objective, but we can remove the distribution assumption for the general $W_2$.\looseness=-1

\paragraph{Individual Fairness Not Always at Odds with Group Fairness.}
Reducing individual unfairness under DistFaiR also reduces group unfairness in most cases (Table~\ref{tab:full_results_main}), as averaged across test splits, even without imposing group-level constraints.  While group unfairness does increase in some cases, the degree of change cannot exceed a specific limit (upto individual unfairness) as per our theoretical findings.  
We also see similar trends on a standard group fairness metric,  EUR~\cite{morik2020controlling} (see Table~\ref{tab:group-re-ranking} in Appendix). Interestingly, the IAA baseline almost always improves group fairness, though  DistFaiR reduces unfairness to a higher degree on two datasets. We also observe higher variance for group unfairness, potentially due to multiple solutions with same individual but different group unfairness.\looseness=-1

\paragraph{Online vs Offline Optimization.}
We observe that fully offline optimization reduces unfairness equally or more effectively than fully online optimization (Appendix Figure~\ref{fig:online_offline_optimization}). Thus, even if the full set of queries is not known apriori, partial offline optimization could be useful when a subset of queries is available. Experimentally, variance in online fairness is lower when optimizing for divergences beyond mean-based differences (\texttt{rateMDs} dataset; Figure~\ref{fig:online_fairness}) over time.\looseness=-1

\paragraph{Fairness Metrics are Sensitive to Query Polarity.}
\label{sec:query_polarity}
In Figure~\ref{fig:varying_theta} (a) and (b), we compute the relative change between fairness metrics measured with and without query polarity, averaged across tuning splits. We observe that all fairness metrics, for both individual and group fairness, are sensitive to query polarity. When the relative change is positive, this indicates fairwashing: rankings seem more fair than they actually are.
We observe that fairwashing occurs, especially for group fairness metrics. Thus, if one relies on the query polarity agnostic metrics, conclusions regarding the (un)fairness of the rankings would be incorrect. It may also be important to consider divergence measures beyond difference in means to avoid systematic under-ranking of a specific group across queries.\looseness=-1

\paragraph{Ranking Quality and Fairness Tradeoff.}
We study the variation in fairness across thresholds of allowable ranking quality loss ($\theta$) in the ILP optimization. Lower unfairness is observed at lower $\theta$ for the polarity-aware re-ranking  (Figure~\ref{fig:varying_theta} (c)), indicating a ranking quality and fairness tradeoff. Additionally, polarity agnostic re-ranking performance leads to higher (worse) unfairness than when query polarity is used. This matches our discussion that fairness metrics are sensitive to query polarity. Experimentally, higher standard deviation is observed in polarity-aware ranking. We also observe similar trends on the \texttt{rateMDs} dataset (Appendix Figure~\ref{fig:tradeoffs_ratemd}).

Importantly, in many real-world applications, different queries may have multiple differing real-world properties beyond polarity.
Accordingly, we can generalize our distribution-aware fairness definition to allow multiple query properties as a vector, where multiple properties form a multi-dimensional distribution. Initial results with this setup for the synthetic datasets are in the Appendix~\ref{sec:multiple_properties}.

\vspace{-0.7em}

\section{Conclusions}
In this paper, we propose a new distribution-aware divergence-based metric, DistFaiR, for amortized fairness measurement. We identify metrics under DistFaiR with the useful property that group unfairness is upper bounded by individual unfairness. We show that we can reduce individual and group unfairness under DistFaiR for different choices of divergence measures. We emphasize query polarity as a crucial yet overlooked aspect in fair-ranking literature, noting that neglecting polarity can result in fairwashing. We also empirically demonstrate fairwashing effects due to a lack of query polarity consideration and propose/evaluate a method to mitigate this effect. 

Our work has some limitations. For example, we assume a position bias model of attention. However, we note that this assumption can be relaxed to consider more complex user attention patterns under our framework, with some modifications made to the cumulative attention formulation. We also make normative assumptions that the distribution of attention should be close to that of relevance. However, a different link function may be more appropriate~\cite{saito2022fair}. Additionally, scores allotted to minority groups may be under-estimates of their true value~\cite{pierson2021algorithmic,krieg2022perceived} and may need to be pre-processed ~\cite{liao2023social}.  Importantly, there may not be purely technical fixes for operationalizing real-world fair ranking~\cite{gichoya2021equity}. Our approach, we believe, is a step towards reducing the scale of such issues.

\bibliographystyle{plain}
\balance
\bibliography{sample-base}

\begin{thebibliography}{10}

\bibitem{achterberg2019s}
Tobias Achterberg.
\newblock What’s new in gurobi 9.0.
\newblock {\em Webinar Talk url: https://www. gurobi. com/wp-content/uploads/2019/12/Gurobi-90-Overview-Webinar-Slides-1. pdf}, 2019.

\bibitem{achterberg2020presolve}
Tobias Achterberg, Robert~E Bixby, Zonghao Gu, Edward Rothberg, and Dieter Weninger.
\newblock Presolve reductions in mixed integer programming.
\newblock {\em INFORMS Journal on Computing}, 32(2):473--506, 2020.

\bibitem{aivodji2019fairwashing}
Ulrich A{\"\i}vodji, Hiromi Arai, Olivier Fortineau, S{\'e}bastien Gambs, Satoshi Hara, and Alain Tapp.
\newblock Fairwashing: the risk of rationalization.
\newblock In {\em International Conference on Machine Learning}, pages 161--170. PMLR, 2019.

\bibitem{altman2005ranking}
Alon Altman and Moshe Tennenholtz.
\newblock Ranking systems: the pagerank axioms.
\newblock In {\em Proceedings of the 6th ACM conference on Electronic commerce}, pages 1--8, 2005.

\bibitem{balagopalan2023role}
Aparna Balagopalan, Abigail~Z Jacobs, and Asia~J Biega.
\newblock The role of relevance in fair ranking.
\newblock In {\em Proceedings of the 46th International ACM SIGIR Conference on Research and Development in Information Retrieval}, pages 2650--2660, 2023.

\bibitem{barbieri-etal-2020-tweeteval}
Francesco Barbieri, Jose Camacho-Collados, Luis Espinosa~Anke, and Leonardo Neves.
\newblock {T}weet{E}val: Unified benchmark and comparative evaluation for tweet classification.
\newblock In {\em Findings of the Association for Computational Linguistics: EMNLP 2020}, pages 1644--1650, Online, November 2020. Association for Computational Linguistics.

\bibitem{beutel2019fairness}
Alex Beutel, Jilin Chen, Tulsee Doshi, Hai Qian, Li~Wei, Yi~Wu, Lukasz Heldt, Zhe Zhao, Lichan Hong, Ed~H Chi, et~al.
\newblock Fairness in recommendation ranking through pairwise comparisons.
\newblock In {\em Proceedings of the 25th ACM SIGKDD International Conference on Knowledge Discovery \& Data Mining}, pages 2212--2220, 2019.

\bibitem{biega2018equity}
Asia~J Biega, Krishna~P Gummadi, and Gerhard Weikum.
\newblock Equity of attention: Amortizing individual fairness in rankings.
\newblock In {\em SIGIR}, pages 405--414, 2018.

\bibitem{bixby2007progress}
Robert Bixby and Edward Rothberg.
\newblock Progress in computational mixed integer programming--a look back from the other side of the tipping point.
\newblock {\em Annals of Operations Research}, 149(1):37, 2007.

\bibitem{bower2020individually}
Amanda Bower, Hamid Eftekhari, Mikhail Yurochkin, and Yuekai Sun.
\newblock Individually fair rankings.
\newblock In {\em ICLR}, 2020.

\bibitem{burke2017multisided}
Robin Burke.
\newblock Multisided fairness for recommendation.
\newblock {\em arXiv preprint arXiv:1707.00093}, 2017.

\bibitem{castillo2019fairness}
Carlos Castillo.
\newblock Fairness and transparency in ranking.
\newblock In {\em Acm sigir forum}, volume~52, pages 64--71. ACM New York, NY, USA, 2019.

\bibitem{chen2018investigating}
Le~Chen, Ruijun Ma, Anik{\'o} Hann{\'a}k, and Christo Wilson.
\newblock Investigating the impact of gender on rank in resume search engines.
\newblock In {\em Proceedings of the 2018 chi conference on human factors in computing systems}, pages 1--14, 2018.

\bibitem{chuklin2015click}
A~Chuklin, I~Markov, and M~de Rijke.
\newblock Click models for web search.
\newblock {\em Synthesis lectures on information concepts, retrieval, and services}, 7(3):1--115, 2015.

\bibitem{clarke2020overview}
Charles~LA Clarke, Saira Rizvi, Mark~D Smucker, Maria Maistro, and Guido Zuccon.
\newblock Overview of the trec 2020 health misinformation track.
\newblock In {\em TREC}, 2020.

\bibitem{craswell2008experimental}
Nick Craswell, Onno Zoeter, Michael Taylor, and Bill Ramsey.
\newblock An experimental comparison of click position-bias models.
\newblock In {\em WSDM}, pages 87--94, 2008.

\bibitem{de2023unfair}
Tim De~Jonge and Djoerd Hiemstra.
\newblock Unfair: Search engine manipulation, undetectable by amortized inequity.
\newblock In {\em Proceedings of the 2023 ACM Conference on Fairness, Accountability, and Transparency}, pages 830--839, 2023.

\bibitem{diaz2020evaluating}
Fernando Diaz, Bhaskar Mitra, Michael~D Ekstrand, Asia~J Biega, and Ben Carterette.
\newblock Evaluating stochastic rankings with expected exposure.
\newblock In {\em CIKM}, pages 275--284, 2020.

\bibitem{dwork2012fairness}
Cynthia Dwork, Moritz Hardt, Toniann Pitassi, Omer Reingold, and Richard Zemel.
\newblock Fairness through awareness.
\newblock In {\em Proceedings of the 3rd innovations in theoretical computer science conference}, pages 214--226, 2012.

\bibitem{trec-fair-ranking-2021}
Michael~D. Ekstrand, Graham McDonald, Amifa Raj, and Isaac Johnson.
\newblock Overview of the trec 2021 fair ranking track.
\newblock In {\em The Thirtieth Text REtrieval Conference (TREC 2021) Proceedings}, 2022.

\bibitem{fabbri2020effect}
Francesco Fabbri, Francesco Bonchi, Ludovico Boratto, and Carlos Castillo.
\newblock The effect of homophily on disparate visibility of minorities in people recommender systems.
\newblock In {\em Proceedings of the International AAAI Conference on Web and Social Media}, volume~14, pages 165--175, 2020.

\bibitem{flanigan2021fair}
Bailey Flanigan, Paul G{\"o}lz, Anupam Gupta, Brett Hennig, and Ariel~D Procaccia.
\newblock Fair algorithms for selecting citizens’ assemblies.
\newblock {\em Nature}, 596(7873):548--552, 2021.

\bibitem{gao2023chat}
Yunfan Gao, Tao Sheng, Youlin Xiang, Yun Xiong, Haofen Wang, and Jiawei Zhang.
\newblock Chat-rec: Towards interactive and explainable llms-augmented recommender system.
\newblock {\em arXiv preprint arXiv:2303.14524}, 2023.

\bibitem{garcia2021maxmin}
David Garc{\'\i}a-Soriano and Francesco Bonchi.
\newblock Maxmin-fair ranking: individual fairness under group-fairness constraints.
\newblock In {\em Proceedings of the 27th ACM SIGKDD Conference on Knowledge Discovery \& Data Mining}, pages 436--446, 2021.

\bibitem{geyik2019fairness}
Sahin~Cem Geyik, Stuart Ambler, and Krishnaram Kenthapadi.
\newblock Fairness-aware ranking in search \& recommendation systems with application to {LinkedIn} talent search.
\newblock In {\em Proceedings of the 25th acm sigkdd international conference on knowledge discovery \& data mining}, pages 2221--2231, 2019.

\bibitem{gichoya2021equity}
Judy~Wawira Gichoya, Liam~G McCoy, Leo~Anthony Celi, and Marzyeh Ghassemi.
\newblock Equity in essence: a call for operationalising fairness in machine learning for healthcare.
\newblock {\em BMJ health \& care informatics}, 28(1), 2021.

\bibitem{gorantla2021problem}
Sruthi Gorantla, Amit Deshpande, and Anand Louis.
\newblock On the problem of underranking in group-fair ranking.
\newblock In {\em International Conference on Machine Learning}, pages 3777--3787. PMLR, 2021.

\bibitem{gorantla2023sampling}
Sruthi Gorantla, Anay Mehrotra, Amit Deshpande, and Anand Louis.
\newblock Sampling individually-fair rankings that are always group fair.
\newblock {\em arXiv preprint arXiv:2306.11964}, 2023.

\bibitem{hajian2016algorithmic}
Sara Hajian, Francesco Bonchi, and Carlos Castillo.
\newblock Algorithmic bias: From discrimination discovery to fairness-aware data mining.
\newblock In {\em Proceedings of the 22nd ACM SIGKDD international conference on knowledge discovery and data mining}, pages 2125--2126, 2016.

\bibitem{hardt2022performative}
Moritz Hardt, Meena Jagadeesan, and Celestine Mendler-D{\"u}nner.
\newblock Performative power.
\newblock {\em Advances in Neural Information Processing Systems}, 35:22969--22981, 2022.

\bibitem{heuss2022fairness}
Maria Heuss, Fatemeh Sarvi, and Maarten de~Rijke.
\newblock Fairness of exposure in light of incomplete exposure estimation.
\newblock In {\em Proceedings of the 45th International ACM SIGIR Conference on Research and Development in Information Retrieval}, pages 759--769, 2022.

\bibitem{hou2023large}
Yupeng Hou, Junjie Zhang, Zihan Lin, Hongyu Lu, Ruobing Xie, Julian McAuley, and Wayne~Xin Zhao.
\newblock Large language models are zero-shot rankers for recommender systems.
\newblock {\em arXiv preprint arXiv:2305.08845}, 2023.

\bibitem{jarvelin2002cumulated}
Kalervo J{\"a}rvelin and Jaana Kek{\"a}l{\"a}inen.
\newblock Cumulated gain-based evaluation of ir techniques.
\newblock {\em ACM Transactions on Information Systems (TOIS)}, 20(4):422--446, 2002.

\bibitem{joachims2021recommendations}
Thorsten Joachims, Ben London, Yi~Su, Adith Swaminathan, and Lequn Wang.
\newblock Recommendations as treatments.
\newblock {\em AI Magazine}, 42(3):19--30, 2021.

\bibitem{kanamori2014scale}
Takafumi Kanamori.
\newblock Scale-invariant divergences for density functions.
\newblock {\em Entropy}, 16(5):2611--2628, 2014.

\bibitem{krieg2022perceived}
Klara Krieg, Emilia Parada-Cabaleiro, Markus Schedl, and Navid Rekabsaz.
\newblock Do perceived gender biases in retrieval results affect relevance judgements?
\newblock In {\em International Workshop on Algorithmic Bias in Search and Recommendation}, pages 104--116. Springer, 2022.

\bibitem{lewandowski1986algorithmic}
JL~Lewandowski, CL~Liu, and Jane W.-S. Liu.
\newblock An algorithmic proof of a generalization of the birkhoff-von neumann theorem.
\newblock {\em Journal of Algorithms}, 7(3):323--330, 1986.

\bibitem{liao2023social}
Yiqiao Liao and Parinaz Naghizadeh.
\newblock Social bias meets data bias: The impacts of labeling and measurement errors on fairness criteria.
\newblock In {\em Proceedings of the AAAI Conference on Artificial Intelligence}, volume~37, pages 8764--8772, 2023.

\bibitem{marchand2002cutting}
Hugues Marchand, Alexander Martin, Robert Weismantel, and Laurence Wolsey.
\newblock Cutting planes in integer and mixed integer programming.
\newblock {\em Discrete Applied Mathematics}, 123(1-3):397--446, 2002.

\bibitem{mehrotra2022fair}
Anay Mehrotra and Nisheeth Vishnoi.
\newblock Fair ranking with noisy protected attributes.
\newblock {\em Advances in Neural Information Processing Systems}, 35:31711--31725, 2022.

\bibitem{mendler2024engine}
Celestine Mendler-D{\"u}nner, Gabriele Carovano, and Moritz Hardt.
\newblock An engine not a camera: Measuring performative power of online search.
\newblock {\em arXiv preprint arXiv:2405.19073}, 2024.

\bibitem{morik2020controlling}
Marco Morik, Ashudeep Singh, Jessica Hong, and Thorsten Joachims.
\newblock Controlling fairness and bias in dynamic learning-to-rank.
\newblock In {\em SIGIR}, pages 429--438, 2020.

\bibitem{naghiaei2022cpfair}
Mohammadmehdi Naghiaei, Hossein~A. Rahmani, and Yashar Deldjoo.
\newblock Cpfair: Personalized consumer and producer fairness re-ranking for recommender systems.
\newblock SIGIR '22, page 770–779, 2022.

\bibitem{patro2022fair}
Gourab~K Patro, Lorenzo Porcaro, Laura Mitchell, Qiuyue Zhang, Meike Zehlike, and Nikhil Garg.
\newblock Fair ranking: a critical review, challenges, and future directions.
\newblock In {\em 2022 ACM Conference on Fairness, Accountability, and Transparency}, pages 1929--1942, 2022.

\bibitem{pierson2021algorithmic}
Emma Pierson, David~M Cutler, Jure Leskovec, Sendhil Mullainathan, and Ziad Obermeyer.
\newblock An algorithmic approach to reducing unexplained pain disparities in underserved populations.
\newblock {\em Nature Medicine}, 27(1):136--140, 2021.

\bibitem{raj2022measuring}
Amifa Raj and Michael~D Ekstrand.
\newblock Measuring fairness in ranked results: An analytical and empirical comparison.
\newblock In {\em Proceedings of the 45th International ACM SIGIR Conference on Research and Development in Information Retrieval}, pages 726--736, 2022.

\bibitem{rastogi2024fairness}
Richa Rastogi and Thorsten Joachims.
\newblock Fairness in ranking under disparate uncertainty.
\newblock In {\em Proceedings of the 4th ACM Conference on Equity and Access in Algorithms, Mechanisms, and Optimization}, pages 1--31, 2024.

\bibitem{robertson1977probability}
Stephen~E Robertson.
\newblock The probability ranking principle in ir.
\newblock {\em J Doc}, 1977.

\bibitem{saito2022fair}
Yuta Saito and Thorsten Joachims.
\newblock Fair ranking as fair division: Impact-based individual fairness in ranking.
\newblock In {\em Proceedings of the 28th ACM SIGKDD Conference on Knowledge Discovery and Data Mining}, pages 1514--1524, 2022.

\bibitem{sanner2023large}
Scott Sanner, Krisztian Balog, Filip Radlinski, Ben Wedin, and Lucas Dixon.
\newblock Large language models are competitive near cold-start recommenders for language-and item-based preferences.
\newblock In {\em Proceedings of the 17th ACM conference on recommender systems}, pages 890--896, 2023.

\bibitem{sarvi2022understanding}
Fatemeh Sarvi, Maria Heuss, Mohammad Aliannejadi, Sebastian Schelter, and Maarten de~Rijke.
\newblock Understanding and mitigating the effect of outliers in fair ranking.
\newblock In {\em Proceedings of the Fifteenth ACM International Conference on Web Search and Data Mining}, pages 861--869, 2022.

\bibitem{schnabel2016recommendations}
Tobias Schnabel, Adith Swaminathan, Ashudeep Singh, Navin Chandak, and Thorsten Joachims.
\newblock Recommendations as treatments: Debiasing learning and evaluation.
\newblock In {\em ICML}, pages 1670--1679, 2016.

\bibitem{singh2018fairness}
Ashudeep Singh and Thorsten Joachims.
\newblock Fairness of exposure in rankings.
\newblock In {\em KDD}, pages 2219--2228, 2018.

\bibitem{singh2019policy}
Ashudeep Singh and Thorsten Joachims.
\newblock Policy learning for fairness in ranking.
\newblock {\em Advances in neural information processing systems}, 32, 2019.

\bibitem{thawani2019online}
Avijit Thawani, Michael~J Paul, Urmimala Sarkar, and Byron~C Wallace.
\newblock Are online reviews of physicians biased against female providers?
\newblock In {\em Machine Learning for Healthcare Conference}, pages 406--423. PMLR, 2019.

\bibitem{wang2021user}
Lequn Wang and Thorsten Joachims.
\newblock User fairness, item fairness, and diversity for rankings in two-sided markets.
\newblock In {\em Proceedings of the 2021 ACM SIGIR International Conference on Theory of Information Retrieval}, pages 23--41, 2021.

\bibitem{wang2018position}
X~Wang, N~Golbandi, M~Bendersky, D~Metzler, and M~Najork.
\newblock Position bias estimation for unbiased learning to rank in personal search.
\newblock In {\em WSDM}, 2018.

\bibitem{wolsey1999integer}
Laurence~A Wolsey and George~L Nemhauser.
\newblock {\em Integer and combinatorial optimization}, volume~55.
\newblock John Wiley \& Sons, 1999.

\bibitem{wu2022joint}
Haolun Wu, Bhaskar Mitra, Chen Ma, Fernando Diaz, and Xue Liu.
\newblock Joint multisided exposure fairness for recommendation.
\newblock In {\em Proceedings of the 45th International ACM SIGIR Conference on Research and Development in Information Retrieval}, pages 703--714, 2022.

\bibitem{xu2024fairsync}
Chen Xu, Jun Xu, Yiming Ding, Xiao Zhang, and Qi~Qi.
\newblock Fairsync: Ensuring amortized group exposure in distributed recommendation retrieval.
\newblock In {\em Proceedings of the ACM on Web Conference 2024}, pages 1092--1102, 2024.

\bibitem{yang2023vertical}
Tao Yang, Zhichao Xu, and Qingyao Ai.
\newblock Vertical allocation-based fair exposure amortizing in ranking.
\newblock In {\em Proceedings of the Annual International ACM SIGIR Conference on Research and Development in Information Retrieval in the Asia Pacific Region}, pages 234--244, 2023.

\bibitem{yang2023change}
Yuzhe Yang, Haoran Zhang, Dina Katabi, and Marzyeh Ghassemi.
\newblock Change is hard: a closer look at subpopulation shift.
\newblock In {\em Proceedings of the 40th International Conference on Machine Learning}, pages 39584--39622, 2023.

\bibitem{zehlike2017fa}
Meike Zehlike, Francesco Bonchi, Carlos Castillo, Sara Hajian, Mohamed Megahed, and Ricardo Baeza-Yates.
\newblock Fa* ir: A fair top-k ranking algorithm.
\newblock In {\em Proceedings of the 2017 ACM on Conference on Information and Knowledge Management}, pages 1569--1578, 2017.

\bibitem{zehlike2020reducing}
Meike Zehlike and Carlos Castillo.
\newblock Reducing disparate exposure in ranking: A learning to rank approach.
\newblock In {\em Proceedings of The Web Conference 2020}, pages 2849--2855, 2020.

\bibitem{zehlike2021fairness}
Meike Zehlike, Ke~Yang, and Julia Stoyanovich.
\newblock Fairness in ranking: A survey.
\newblock {\em arXiv preprint arXiv:2103.14000}, 2021.

\bibitem{zehlike2022fairness}
Meike Zehlike, Ke~Yang, and Julia Stoyanovich.
\newblock Fairness in ranking, part i: Score-based ranking.
\newblock {\em ACM Computing Surveys}, 55(6):1--36, 2022.

\bibitem{zeide2022silicon}
Elana Zeide.
\newblock The silicon ceiling: How algorithmic assessments construct an invisible barrier to opportunity.
\newblock {\em UMKC Law Rev.}, 2022.

\bibitem{zhuang2023open}
Shengyao Zhuang, Bing Liu, Bevan Koopman, and Guido Zuccon.
\newblock Open-source large language models are strong zero-shot query likelihood models for document ranking.
\newblock In {\em Findings of the Association for Computational Linguistics: EMNLP 2023}, pages 8807--8817, 2023.

\end{thebibliography}

\appendix

\section{Glossary}
\begin{table}[!htp]
\centering
\adjustbox{max width=\linewidth}{%
\begin{tabular}{lrr}\toprule
\textbf{Term} &\textbf{Definition} \\\midrule
Attention &\cellcolor[HTML]{fffdfa}search exposure, or likelihood of being viewed on a search result \\
Relevance &\cellcolor[HTML]{fffdfa}score per individual per query, where a higher score indicates better rankability \\
Dataset &\cellcolor[HTML]{fffdfa}set of individuals, with relevance scores available per individual per query \\
Cumulative attention &attention over a sequence of queries \\
\bottomrule
\end{tabular}}
\caption{Glossary of terms utilized in the paper.}\label{tab:glossary}

\end{table}

\section{Information Loss with Expectation-only Approaches} \label{sec:dist_example}
We observe that critical information about distributions of attention may be missing from prior formulations e.g., as shown in Figure~\ref{fig:distribution_reliability}, where the mean attention may match the mean relevance scores for an individual, but the distributions may be very dissimilar. Consider two distributions $A$ and $R$ defined as follows

\[
A = \Ncal(0, \sigma^2) \quad \quad \text{ and } \quad \quad R = 0.5\Ncal(-\mu_R, \sigma^2) + 0.5\Ncal(\mu_R, \sigma^2),
\]

Also, define $$\Tilde{R} = \Ncal(0, \sigma^2)$$

Clearly $\mu_{A} - \mu_{R} = \mu_{A} - \mu_{\Tilde{R}} = 0$. However, the distribution between attention and relevance is clearly not the same. Then, for $\sigma_{\tilde{R}} = \sigma_A < \sigma_R$, attention is spread out less broadly across individuals as relevance. Thus, attention is much more concentrated for some individuals, while relevance is not concentrated within the same individuals. Importantly, fairness metrics that only consider means would consider this setting fair.

Then 
\begin{align}
D_{\text{KL}}(A \| R) = \frac{1}{2} \left[\frac{\sigma_A^2}{\sigma_R^2} + \frac{(\mu_R - \mu_A)^2}{\sigma_R^2} - 1 + \log \frac{\sigma_R^2}{\sigma_A^2}\right].
\end{align}

Clearly, $D_{\text{KL}}(A \| \Tilde{R}) = 0 < D_{\text{KL}}(A \| R)$, better measuring the discrepancy between cumulative attention and relevance, unlike mean distance measures in previous work.

\begin{figure}[htb!]
\centering
\begin{minipage}[htb!]
{0.4\textwidth}
\includegraphics[width=\textwidth]{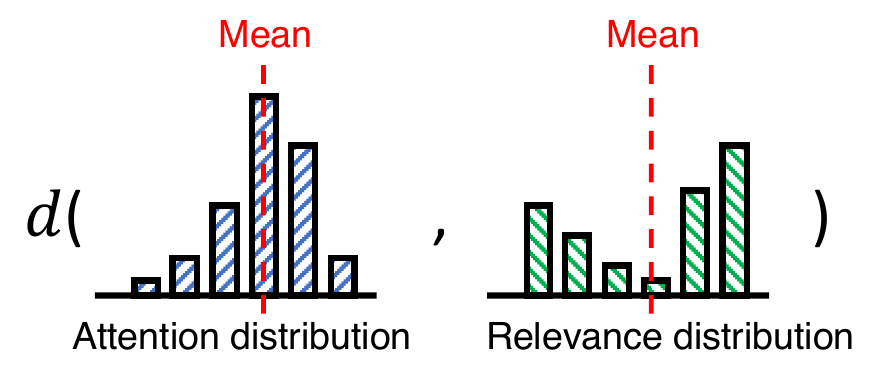}
\caption{Critical information about the distributions of relevance and attention (e.g., the variance) may be missing in such formulations.
}\label{fig:distribution_reliability}
\end{minipage}
\end{figure}

\section{Example of Fairwashing}

As shown in Figure~\ref{fig:fairwashing_amortized_ranking}, past formulations of amortized fair ranking may be prone to fairwashing.

\begin{figure}[htb!]
\centering
\begin{minipage}[htb!]
{0.5\textwidth}
\includegraphics[width=\textwidth,trim={1.5cm 11cm 0 5cm},clip]{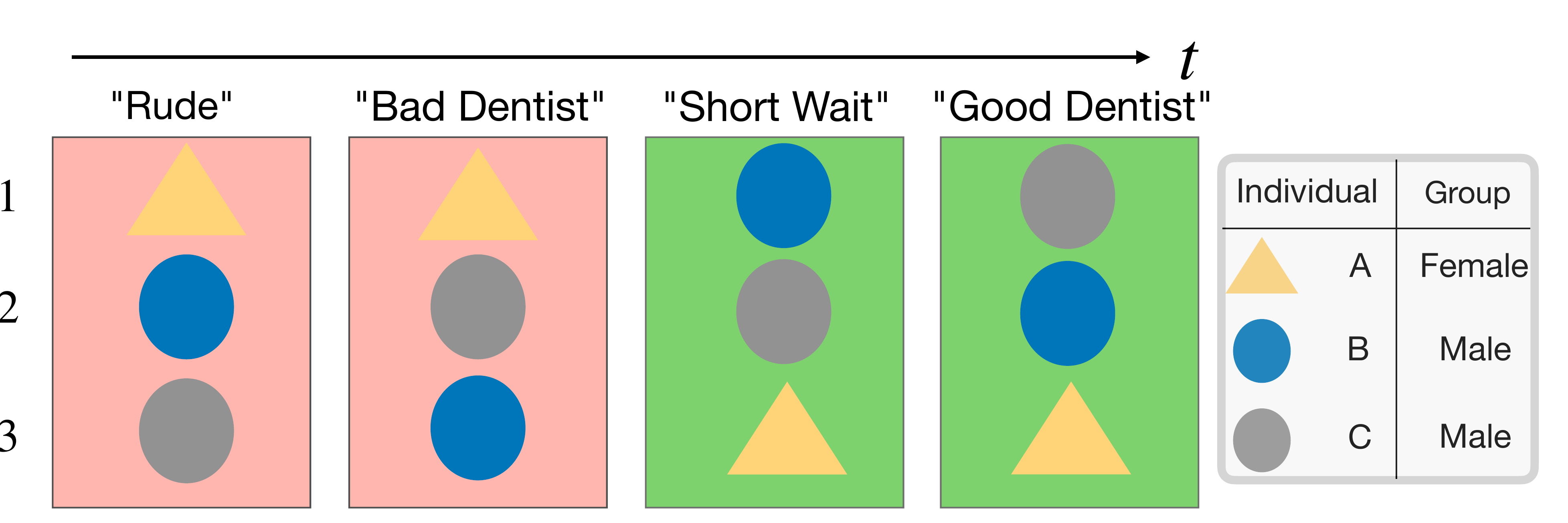}
\caption{\textbf{Past work in amortized fair ranking has ignored the impact of query polarity, and only considered expected cumulative attention}. 
Here, if all individuals are equally relevant, and expected attention scores for ranks 1,2,3 are $\{0.5,0.5,0\}$ respectively, the sequence of queries appear fair because an individual's expected attention accumulated over the four queries is proportional to their relevance. However, we the female doctor is allocated attention only for the queries with negative polarity (``rude",``bad dentist"). This leads to \emph{fairwashing}.
}\label{fig:fairwashing_amortized_ranking}
\end{minipage}
\end{figure}

\section{Datasets}
\label{sec:datasets}
We utilize four datasets in our experiments. In each dataset, relevance scores are normalized to form a distribution within a ranking. \texttt{fairtrec2021} is licensed under the CC BY-SA 3.0 license. \texttt{rateMDs} is released as a part of open-sourced research publication~\cite{thawani2019online}.

\subsection{Synthetic Hiring Datasets}
Two synthetic datasets are generated, to mimic the example shown in Figure~\ref{fig:fairwashing_amortized_ranking}. The sensitive attribute in each dataset is the sex of the individual being ranked. We generate two versions, one with two unique relevance scores, and one where the relevance is continuous. In the binary relevance dataset (\texttt{synth-binary}) the relevance for individuals is either 1.01 or 0.99, with male individuals having a relevance score of 1.01 for queries with positive polarity, and 0.99 for queries with negative polarity. This represents a nearly uniform relevance distribution. In the continuous dataset (\texttt{synth-cont}), the relevance in each group is sampled from a normal distribution, with mean and standard deviation $1$ and $0.2$ respectively for one group, and $1$ and $0.1$ for the other group per query (i.e., disparate uncertainty for groups)~\cite{rastogi2024fairness}. Each query in the dataset has a polarity value $\in \{-1,1\}$. Note that relevance scores are normalized within each query set before all re-ranking interventions.

\subsection{RateMDs}
We also utilize a healthcare dataset~\cite{thawani2019online} for ranking doctors corresponding to a text query. The sensitive attribute is sex. The original dataset contains reviews corresponding to 6197 profiles of doctors on \href{https://www.ratemds.com/}{RateMDs}. We consider each of these profiles to be a unique individual. We verified that the `name' metadata associated with each of these profiles is unique (including after lower-casing and removing punctuations).  Each ranking then corresponds to a text query such as ``short wait time". The ranking is produced by using a pre-trained LLM model\footnote{\url{https://huggingface.co/cross-encoder/ms-marco-TinyBERT-L-2}} to match the text to reviews of doctors, and order the doctors in decreasing order of ranking score. The text-match score is averaged across all reviews corresponding to a given doctor for each query. Additionally, each query is associated with a polarity $\in \{-1,1\}$, which is annotated based on the sentiment polarity of the query (positive sentiment: 1, negative sentiment: -1). Note that some queries were specific to a speciality: e.g., ``best dentist", and relevance scores were produced for all doctors in the dataset (i.e., including doctors who have a different specialty). Thus, our results produced are highly influenced by the LLM-driven scores. Ideally, the scores allocated to doctors from different specialties will be low. In practice, while we observed some variance in top-selected items for different but similar queries, we do observe that some profiles occupy top positions in several rankings. We leave experiments with varying how the relevance score is produced to future work.\looseness=-1 

\subsection{FairTREC2021}
In this dataset, the items correspond to Wikipedia articles, and the query is the corresponding article domain~\cite{trec-fair-ranking-2021}. We utilize the standard evaluation split\footnote{\url{https://ir-datasets.com/trec-fair-2021.html\#trec-fair-2021/eval}} for experiments, which is a benchmark dataset used in multi-query fair ranking evaluation. The sensitive attribute is geographic location(s) referred to in the data, all of which we categorize into one of three groups (Europe/North America, No continent, or Others) to avoid non-overlapping groups. The query polarity score is a continuous score $\in [-1,1]$, and is produced by a pre-trained sentiment classification model~\cite{barbieri-etal-2020-tweeteval}. Specifically, the polarity score is the sum of the sentiment polarity of the query -- where -1 denotes negative, 0 denotes neutral, and 1 denotes positive -- weighted by predicted probabilities of each polarity class. We use the evaluation split of this dataset for all experiments. Note that the corpus corresponding to the FairTREC2021 dataset contains over 6.3M documents. However, the majority of these documents are not relevant to even a single query in the evaluation set. Thus, we filter the ranking corpus to documents marked as being relevant to at least one query for re-ranking before all experiments. Note that this dataset selection step is a normative choice, made with the assumption that relevance scores are unbiased: i.e., documents not relevant to any query need not be re-ranked for any higher attention. When making similar decisions at scale, it might be important to assess if there is a high degree of under-ranking, e.g., by expanding the query and corpus.

\section{Individual vs Group Fairness}

\subsection{Tail Probability Bounds for Cumulative Attention and Relevance}
\begin{theorem}
Let $X_i^t \sim \text{Bernoulli}(p_i^t)$ and
\[
    X_i = \sum_{t\in \Tcal} X_i^t.
\]
The expected value of $X_i$ is given by:
\[
    \EE[X_i] = \sum_{t\in \Tcal} p_i^t.
\]
\vspace{1em}
Then, for any $\delta > 0$, we have the following:
\[
    \mathbb{P}\left(|X_i - \EE[X_i]| \geq \delta \EE[X_i]\right) \leq 2\exp\left(-\frac{\delta^2 \mathbb{E}[X_i]}{2 + \delta}\right).
\]
\end{theorem}
\begin{proof}
    Assume that $X_i^t$'s are independent for different $t$ and observe that the domain of random variable $X_i^t$ is $\{0, 1\}$, i.e., bounded and non-negative. Using Chernoff bounds for the sum of independent Bernoulli random variables we have upper and tail bounds, respectively:

    \[
        P(X_i \geq (1 + \delta)\EE[X_i]) \leq \exp\left(-\frac{\delta^2 \EE[X_i] }{2 + \delta}\right),
    \]
    \[
        P(X_i \leq (1 - \delta)\EE[X_i]) \leq \exp\left(-\frac{\delta^2 \EE[X_i]}{2}\right).
    \]
    Applying a union bound for both the upper and lower tails, we have:
    \[
        P\left(|X_i - \EE[X_i] | \geq \delta \EE[X_i]\right) \leq 2 \exp\left(-\frac{\delta^2 \EE[X_i]}{2 + \delta}\right).
\]
\end{proof}
\subsection{Proof of Lemmas}
\begin{lemma}
Define the following:
    \begin{align*}
        D_{L_1}(P \| Q) &= |\mu_P - \mu_Q|
    \end{align*}

    $D_{L_1}$ satisfies definition \ref{def:divergence} for $P$ and $Q$ when $\mu_P$ and $\mu_Q$ are sufficient statistics for their respective distributions. 
\end{lemma}
\begin{proof}
We prove that $D_{L_1}(P \| Q) = |\mu_P - \mu_Q|$ satisfy the following properties:

\paragraph{1. Non-negativity:} \\
For $D_{L_1}(P \| Q)$, the expressions involve absolute values, which are non-negative by definition. Thus,
\begin{align}
&D_{L_1}(P \| Q) = |\mu_P - \mu_Q| \geq 0.
\end{align}

\paragraph{2. Positivity:} \\
For $D_{L_1}(P \| Q) = |\mu_P - \mu_Q|$, we have $D_{L_1}(P \| Q) = 0$ if and only if $\mu_P = \mu_Q$. Since $\mu_P$ and $\mu_Q$ are sufficient statistics, $\mu_P = \mu_Q$ implies $P = Q$, and conversely, if $P = Q$, then $\mu_P = \mu_Q$.

\paragraph{3. Subadditivity:}
For \( D_{L_1}(P \| Q) \), we verify subadditivity under convolution in terms of the mean-difference divergence:
\[
D_{L_1}(P_1 \circ P_2 \| R_1 \circ R_2) \leq D_{L_1}(P_1 \| R_1 ) + D_{L_1}(P_2 \| R_2).
\]
Since expectation is linear under convolution:
\[
\mu_{P_1 \circ P_2} = \mu_{P_1} + \mu_{P_2},  \quad \mu_{R_1 \circ R_2} = \mu_{R_1} + \mu_{R_2},
\]
the mean-difference divergence simplifies to:
\[
|\mu_{P_1 \circ P_2} - \mu_{R_1 \circ R_2}|
= |(\mu_{P_1} + \mu_{P_2}) - (\mu_{R_1} + \mu_{R_2})|.
\]
Rewriting,
\[
|\mu_{P_1} - \mu_{R_1} + \mu_{P_2} - \mu_{R_2}|
\]
Applying the triangle inequality,
\[
|\mu_{P_1} - \mu_{R_1} + \mu_{P_2} - \mu_{R_2}| \leq |\mu_{P_1} - \mu_{R_1}| + |\mu_{P_2} - \mu_{R_2}|.
\]
Thus, subadditivity holds:
\[
D_{L_1}(P_1 \circ P_2 \| R_1 \circ R_2) \leq D_{L_1}(P_1  \| R_2) + D_{L_1}(P_2 \| R_2).
\]

\paragraph{4. Scaling over averages:}
For $D_{L_1}(P \| Q)$, scaling over averages refers to how the divergence behaves when comparing averages (means) of distributions.

\[
\left| \frac{\mu_{P_1} + \mu_{P_2}}{2} - \frac{\mu_{Q_1} + \mu_{Q_2}}{2} \right| \leq \frac{|\mu_{P_1} - \mu_{Q_1}| + |\mu_{P_2} - \mu_{Q_2}|}{2}
\]
Again, this holds due to the triangle inequality for absolute values.

\paragraph{5. Positive homogeneity with degree one:}
For $D_{L_1}(P \| Q)$, $\alpha > 0$:
\[
D_{L_1}(\alpha P \| \alpha Q) = |\alpha \mu_P - \alpha \mu_Q| = \alpha |\mu_P - \mu_Q| = \alpha D_{L_1}(P \| Q),
\] where here scaling denotes scaling the random variables corresponding to the distribution.
Thus, positive homogeneity holds.

\end{proof}

\subsection{Individual Fairness Upper-Bounds Group Fairness}
\label{ref:sec_proof}
\begin{theorem}
    For any jointly convex DistFaiR divergence that is subadditive under the convolution operation, positively homogeneous with degree $s$, and scales under averages, amortized group fairness is upper-bounded by amortized individual fairness. Specifically, we have the following inequality:
    \begin{align}
        \max_{g_k \in \Gcal} D(A_{g_k} \| R_{g_k}) \leq   
 \max_{i \in \Dcal} D(A_i \| R_i) \hspace{0.5cm},
    \end{align}
where $A$ and $R$ are distributions that denote attention and relevance, respectively, individuals $i\in \{1, \ldots, n\}$, and $g_k$ denotes the set of individuals $i$ that belong to group $k$.
\end{theorem}

\begin{proof}
\newcommand{\Xgk}{X_{g_{k}}}
\newcommand{\Ygk}{Y_{g_{k}}}
\newcommand{\sumXgk}{\frac{1}{|g_k|}\sum_{i\in g_k} X_{i}}
\newcommand{\sumYgk}{\frac{1}{|g_k|}\sum_{i\in g_k} Y_{i}}
Let $A_i, R_i$ denote the distributions of random variables $X_i, Y_i$, respectively. 

{\bf Assume $D$ is subadditive, positively homogeneous, and scales under averages for transformations to the random variable corresponding to the distributions in each case.}

Denote 
\begin{equation}
    X_{g_k} = \frac{1}{|g_k|}\sum_{i\in g_k} X_i \quad \text{ and } \quad Y_{g_k} = \frac{1}{|g_k|}\sum_{i\in g_k} Y_i,
\end{equation}

such that, by scaling property of $D$,

\begin{equation}
    X_{g_k} \sim A_{g_k} \quad \text{ and } \quad Y_{g_k} \sim R_{g_k}.
\end{equation}

Denote $$X'_i = \frac{1}{|g_k|}X_i \quad \text{ and } \quad Y'_i = \frac{1}{|g_k|}Y_i$$ s.t. $$X'_i \sim A'_i \quad \text{ and } \quad Y'_i \sim R'_i .$$

$A_{g_k} = A'_1 \circ A'_2 \circ \ldots \circ A'_{|g_k|}$ and $R'_{g_k} = R'_1 \circ R'_2 \circ \ldots \circ R'_{|g_k|}$, where $\circ$ denotes convolution. Recall that $X_i$'s and $Y_i$'s are independent.

\begin{align}
    D(A_{g_k}  \| R_{g_k}) &\le \sum_{i\in g_k} D(A'_i \| R'_i) \label{eq:subadditivity} \\
    &= \frac{1}{|g_k|^s}\sum_{i\in g_k} D(A_i \| R_i) \label{eq:homogeneity} \\
    &\le \frac{1}{|g_k|^{s-1}} \max_{i\in g_k} D(A_i \| R_i) \\
    &\le \max_{i\in \Dcal} D(A_i \| R_i) \label{eq:assumption_groups},
\end{align}

where Equation \ref{eq:subadditivity} is a result of subadditivity and Equation \ref{eq:homogeneity} is a result of positive homogeneity~\cite{kanamori2014scale} with degree $s$. Specifically, subadditivity of a divergence under convolution refers to divergences $D$ satisfying the property:
\[
D(P_1 \circ P_2 \| Q_1 \circ Q_2) \leq D(P_1 \| Q_1 ) + D(P_2 \| Q_2).
\] for some distributions $P_1, P_2, Q_1, Q_2$.

Taking the max over all groups,
\begin{equation}
    max_{g_k \in \Gcal} D\big(A_{g_k} \| R_{g_k}\big) \le \max_{i \in \Dcal} D(A_i \| R_i),
\end{equation}
completes the proof. Note that Equation \ref{eq:assumption_groups} holds when $|g_k| \geq 1$, which is an underlying assumption in the fairness measurement. We experimentally validate this theorem in Figure~\ref{fig:indiv_group}.

Note that this bound becomes especially pertinent when there are a large number of groups, and the ranges of group and individual fairness are similar.

\end{proof}

\subsection{Tail Probability Bounds for Polarity-Aware Cumulative Attention and Relevance}
\begin{theorem}
Let $X_i^t \sim \text{Bernoulli}(p_i^t)$ and $\eta(q_t) \in [a_t, b_t]$; $a_t, b_t \in \RR$. With a slight abuse of notation, let $\Tilde{X}_i^t = X_i^t\cdot \eta(q_t) \in [a_t, b_t]$ and
\[
    \Tilde{X}_i = \sum_{t\in \Tcal} \Tilde{X}_i^t,
\]
The expected value of $\Tilde{X}_i$ is given by:
\[
    \mathbb{E}[\Tilde{X}_i] = \sum_{t\in \Tcal} \eta(q_t) \cdot p_i^t.
\]
Then, for any $\delta > 0$, we have the following:
\[
    P\left(|\Tilde{X}_i - \EE[\Tilde{X}_i]| \geq \delta\right) \leq 2\exp\left(-\frac{2\delta^2}{\sum_{t \in \Tcal} (b_t - a_t)^2}\right).
\]
\end{theorem}

\begin{proof}
Assume that $\Tilde{X}_i^t$'s are independent for different $t$ and observe that each $\Tilde{X}_i^t \in [a_t, b_t]$, i.e., bounded. Using Hoeffding's inequality for the sum of independent bounded random variables, we have:

\[
    P\left( \Tilde{X}_i \geq \mathbb{E}[\Tilde{X}_i] + \delta \right) \leq \exp\left(-\frac{2\delta^2}{\sum_{t \in \mathcal{T}} (b_t - a_t)^2}\right),
\]
\[
    P\left( \Tilde{X}_i \leq \mathbb{E}[\Tilde{X}_i] - \delta \right) \leq \exp\left(-\frac{2\delta^2}{\sum_{t \in \mathcal{T}} (b_t - a_t)^2}\right).
\]

By applying a union bound for the upper and lower tails, we get:

\[
    P\left( |\Tilde{X}_i - \mathbb{E}[\Tilde{X}_i]| \geq \delta \right) \leq 2 \exp\left( -\frac{2\delta^2}{\sum_{t \in \mathcal{T}} (b_t - a_t)^2} \right).
\]
\end{proof}

\begin{table*}[
htb!]
\centering
\caption{We show \emph{relative improvement} in fairness post-fair ranking intervention with respect to the original ranking, all with the FIGR baseline. The columns correspond to different fairness measures, while each row corresponds to a dataset.}
\label{tab:figr}
\adjustbox{max width=0.8\linewidth}{%
 \begin{tabular}{llcccccc}
\toprule
Dataset & Method  & \multicolumn{3}{c}{Relative Change in Individual Fairness ($\uparrow$)}  &\multicolumn{3}{c}{Relative Change in Group Fairness ($\uparrow$)} \\
\cmidrule(r){3-5}\cmidrule(r){6-8}\\
 &  & $\Delta$ DistFaiR ($L_1$) & $\Delta$ DistFaiR ($L_2^{var}$) & $\Delta$ DistFaiR ($W_1$)  & $\Delta$ DistFaiR ($L_1$) & $\Delta$ DistFaiR ($L_2^{var}$) & $\Delta$ DistFaiR ($W_1$)  \\

\midrule
\multirow{1}{*}{\texttt{synth-binary}} 
  & FIGR & 0.00\% & 0.00\% & 0.00\% & 0.00\% & 0.00\% & 0.00\%\\
 
\midrule

\multirow{1}{*}{\texttt{synth-cont}}  
& FIGR & 0.00\% & 0.00\% & 0.00\% & 0.00\% & 0.00\% & 0.00\%\\
\midrule

\multirow{1}{*}{\texttt{FairTREC2021}}  
  & FIGR & -3.37\% & -6.18\% & -3.37\% & -164.04\% & -417.48\% & -147.08\%\\

\midrule
\multirow{1}{*}{\texttt{rateMDs}}   
& FIGR & 0.00\% & 0.00\% & 0.00\% & -4.64\% & -13.20\% & 7.99\%\\

\bottomrule
\end{tabular}}
\end{table*}

\section{Xing Dataset}
\label{app:xing_dataset}
The Xing dataset~\cite{zehlike2017fa} contains top-ranked candidates in response to 57 queries submitted on \href{http://xing.com}{Xing}, all in the context of online hiring. Candidate information is anonymized. All candidates are categorized into two groups (male and female) based on their sex. Note that unique IDs are not available for each individual, as all candidate names, pictures etc. were either ``removed or obfuscated by hashing"~\cite{zehlike2017fa}. Further, relevance scores are not available for any candidate or query.  

To use this dataset in our study, we make some assumptions. First, we assume that the top-40 individuals in each ranking are distinct. This leads to a dataset of 57 queries and 2236 individuals. To illustrate the impact of polarity, we also assume that the four specific queries related to software engineering -- `Software Engineer',`Back end Developer',`Front End Developer',`Application Developer' -- have a polarity of one, and the rest have a polarity of zero. This simulates a setup where e.g., the recruiter is searching specifically to hire a software engineer. Since there are no relevance scores in this dataset, we assume that male and female candidates are equally relevant (i.e., equally worthy of attention).

Since an individual, under this setting, can only appear once (i.e., corresponding to a single query), amortized individual fairness is less meaningful. Thus, we only measure group fairness. We measure different group fairness measures including DistFaiR on the full dataset. In Table~\ref{tab:xing_dataset}, we measure proportion of change in (un)fairness when polarity is considered with respect to the polarity agnostic fairness measurement, randomly sampling queries with replacement several times.   Similar to results in the main text, a positive value indicates that unfairness is higher when polarity is considered, or that the rankings were actually more unfair than they appeared. Thus, we find that most group fairness metrics are sensitive to query polarity, and vary across distance measures on the Xing dataset as well.  

\begin{table}[!htp]\centering
\begin{tabular}{c|rr}
\toprule
\textbf{Metric} &\textbf{Group Fairwashing} \\\midrule
$L_1$ &0.16 \\
$L_2^{\text{var}}$ &0.40 \\
$W_1$ &-0.90 \\
DP &0.16 \\
EUR &15.20 \\
IAA &0.16 \\
\bottomrule
\end{tabular}
\caption{Impact of query polarity on fairness measurement for the Xing dataset. Results are averaged over 10 sets, each a random sample of queries with replacement.}\label{tab:xing_dataset}

\end{table}

\section{Integer Linear Programming Formulation}
\label{sec:ilp_appendix}
We rely on the Gurobi software~\cite{achterberg2019s} for each optimization. Within the solver, the optimality gap~\cite{bixby2007progress} is measured during each optimization search iteration. The Gurobi solver also uses routines of branch-and-bound~\cite{wolsey1999integer}, presolve~\cite{achterberg2020presolve}, cutting planes~\cite{marchand2002cutting}, etc. We emphasize that our framework - DistFaiR – can also be applied to continuously relaxed, more computationally efficient formulations. We highlight this as an important direction of future work. In implementation, we introduce an auxilary variable upper-bounding unfairness across all individuals (so worst-case unfairness), which is then minimized. 

Note that we ignore ties in system relevance scores in all cases. We verified that variations in DCG@K on accounting vs not accounting for tied relevance scores are very small (less than 0.0001). In implementation, we also pre-filter and re-rank only top-k individuals in each ranking. Thus,  our results are sensitive to this pre-filtering choice. As also mentioned in prior work~\cite{biega2018equity}, different pre-filtering choices may lead to different degrees of fairness gains. We leave the exploration of different pre-filtering choices to future work.

\section{Individual Fairness Bounds Group Fairness under DistFaiR.}
\label{sec:fairness_bounds}
Theorem~\ref{theorem:indiv_group} shows that group (un)fairness is upper-bounded by individual (un)fairness for some classes of distance functions. In Figure~\ref{fig:indiv_group}, we experimentally validate this by computing group and individual unfairness for the $W_1$ divergence measure.

\begin{figure}[ht!]
    \centering
    \includegraphics[width=0.6\linewidth]{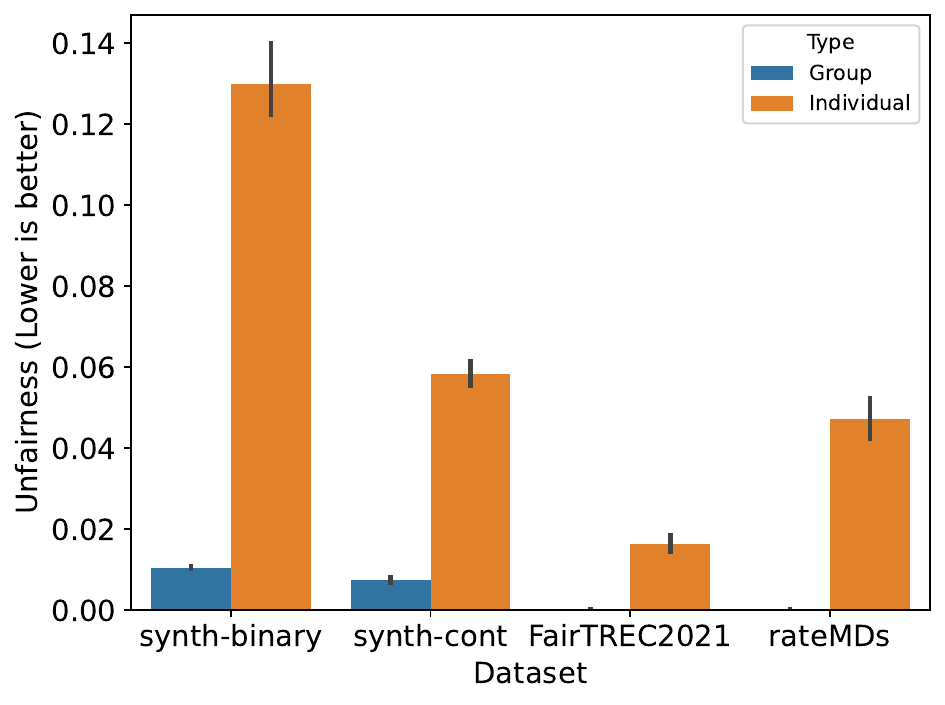}
    \caption{Individual Fairness Bounds Group Fairness under DistFaiR (here, $DistFaiR(W_1)$)}
    \label{fig:indiv_group}
\end{figure}

\section{Note on Query Polarity Aware Ranking}
Note that we can also utilize normalized versions of query polarity, e.g., with softmax-based normalization.  However, normalization in this manner ensures may convert queries with negative polarity to positive -- thus we use unnormalized scores in our experiments.

\section{Additional Baseline: FIGR}

We show performance of FIGR in Table~\ref{tab:figr}. Due to the difference in optimization objective, and the fact that FIGR optimizes fairness per ranking (i.e., proportion-based group and individual unfairness per ranking), we find that this baseline underperforms DistFaiR as well as other exposure-based fairness interventions. We also rely on the official open-sourced implementation of the algorithm for experiments, and only consider binary groups as in the paper (results with more groups showed similar trends). We consider an audit interval length of $k=10$, and observe similar results for different $k$.\looseness=-1

\section{Fairness Over Time}
We observe that variance in online fairness is lower for divergence measures that use higher order moments in Figure~\ref{fig:online_fairness}, on the \texttt{rateMDs} dataset. Verifying this across more divergence measures is thus an interesting direction of future work.
\begin{figure}[htb!]
    \centering
    \includegraphics[width=0.6\linewidth]{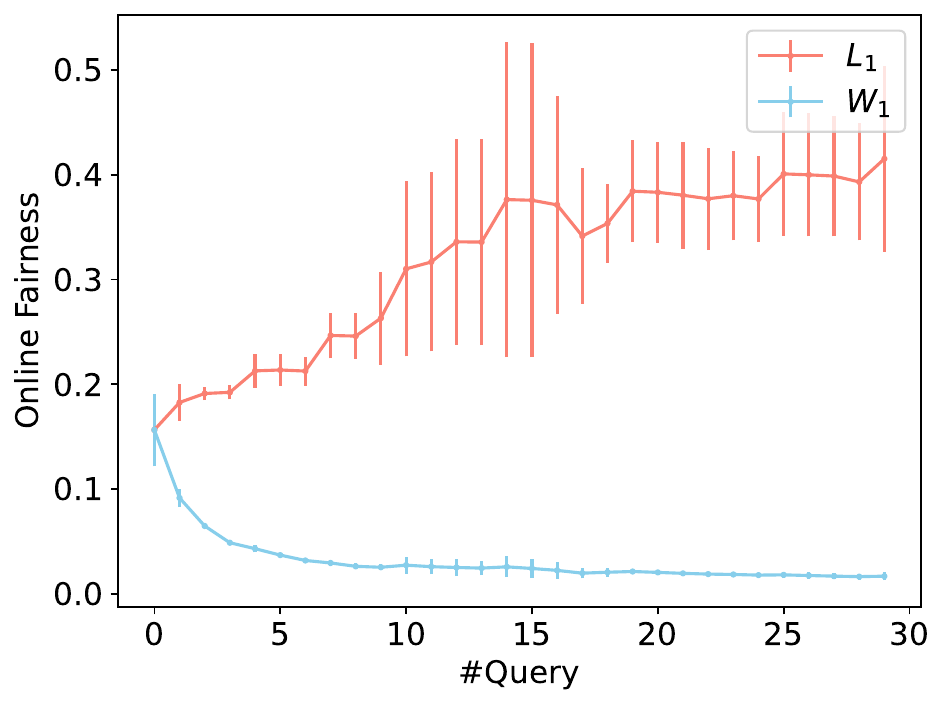}
    \caption{Online fair ranking fairness on \texttt{rateMDs}.}
\label{fig:online_fairness}
\end{figure}

\section{Additional Results: Impact of Re-ranking on Group Fairness}

Re-ranking interventions optimized to improve individual fairness tend to improve or retain EUR as seen in Table~\ref{tab:group-re-ranking} for atleast one divergence measure on three of four datasets. Importantly, DistFaiR underperforms IAA on the IAA individual fairness measurement which makes sense because DistFaiR focuses on worst-case distance between individuals, while IAA focuses on average across individuals. Thus, there are tradeoffs between average and worst-case performance as seen in other fairness contexts~\cite{yang2023change}. Note that we set degree of permissible performance loss (i.e., least possible nDCG) to 80\%, which all methods exceed.

\begin{table}[htb!]
\centering
\caption{\textbf{DistFaiR also improve IAA and EUR in a majority of cases (here, positive, higher is better). However, there are some distance and dataset-dependent variations.}  We show change in fairness when compared to the unconstrained ranking. IAA baseline outperforms DistFaiR on the IAA fairness measurement}\label{tab:group-re-ranking}
\adjustbox{max width=0.7\linewidth}{%

\begin{tabular}{lccccc}
\toprule
Dataset & Baseline & \multicolumn{2}{c}{Fairness} & nDCG@10 \\
\cmidrule(r){3-4}\\
 & & IAA & EUR &  \\
\midrule

\multirow{4}{*}{\texttt{synth-binary}} 
& IAA & \textbf{68.88}\%  & 8.84\% & 100\%\\
 & FoE &  13.31\% & 20.81\% & 100\% \\
 & DistFaiR($L_1$) &  58.32\% & 47.05\%& 100\% \\

 & DistFaiR($L_2^{var}$) &  38.99\% & 	59.05\% & 100\%\\
 & DistFaiR($W_1$) & 43.39\% & \textbf{76.10}\% & 100\%  \\
 
\midrule

\multirow{4}{*}{\texttt{synth-cont}} 
& IAA & \textbf{45.37}\% & \textbf{38.52\%} & 92\%\\
  & FoE &  1.26\% & -139.41\% & 98\% \\
  & DistFaiR($L_1$) &  27.96\% & -36.10\% & 88\% \\

 & DistFaiR($L_2^{var}$) &  34.39\% & 	-39.66\% & 87\%\\
 & DistFaiR($W_1$) & 0.47\% & -125.17\% & 86\%  \\

\midrule

\multirow{5}{*}{\texttt{FairTREC2021}} 
& IAA & \textbf{7.55}\% & 41.50\%  & 100\%\\
& FoE & 4.40\% & 41.07\%  & 100\%\\

  & DistFaiR($L_1$) & 6.12\%  & 45.54\%  & 100\%\\

 & DistFaiR($L_2^{var}$) &  4.45\% & 	43.40\%  & 100\%\\
 & DistFaiR($W_1$) & 6.07\%  & \textbf{46.57\%}  & 100\%\\

\midrule
\multirow{4}{*}{\texttt{rateMDs}} 
& IAA & \textbf{5.67}\% & -7.94\% &  89\% \\
 & FoE & -0.32\% & 26.96\% & 97\% \\

 & DistFaiR($L_1$) &  -0.41\% & \textbf{62.19\%} &  85\%  \\
 & DistFaiR($L_2^{var}$) &  -0.60\% & 39.99\%&  84\% \\
 & DistFaiR($W_1$) &  -1.68\% & 57.41\%&  84\% \\

\bottomrule
\end{tabular}
}
\end{table}

\begin{figure}
    \centering
    \includegraphics[width=0.6\linewidth]{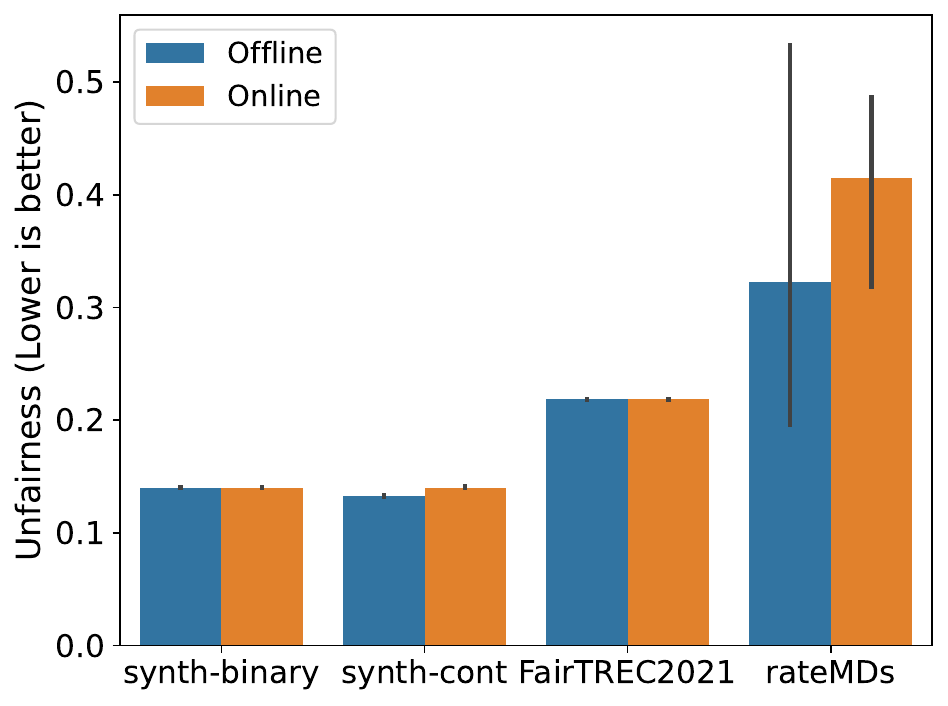}
    \caption{Online fair ranking -- where queries arrive one after the other for ranking -- underperforms or performs similarly as offline fair ranking where the whole set of queries is known apriori. However, the degree of difference is small. The $L_1$ distance function is used for this figure.}
    \label{fig:online_offline_optimization}

\end{figure}

\section{Performance Trade-offs with Query Polarity}
Similar to Figure~\ref{fig:varying_theta}, we observe both performance-fairness trade-offs, and the impact of polarity on fairness optimization on the \texttt{rateMDs} (see Figure~\ref{fig:tradeoffs_ratemd}) dataset. 
\begin{figure}
    \centering
    \includegraphics[width=0.6\linewidth]{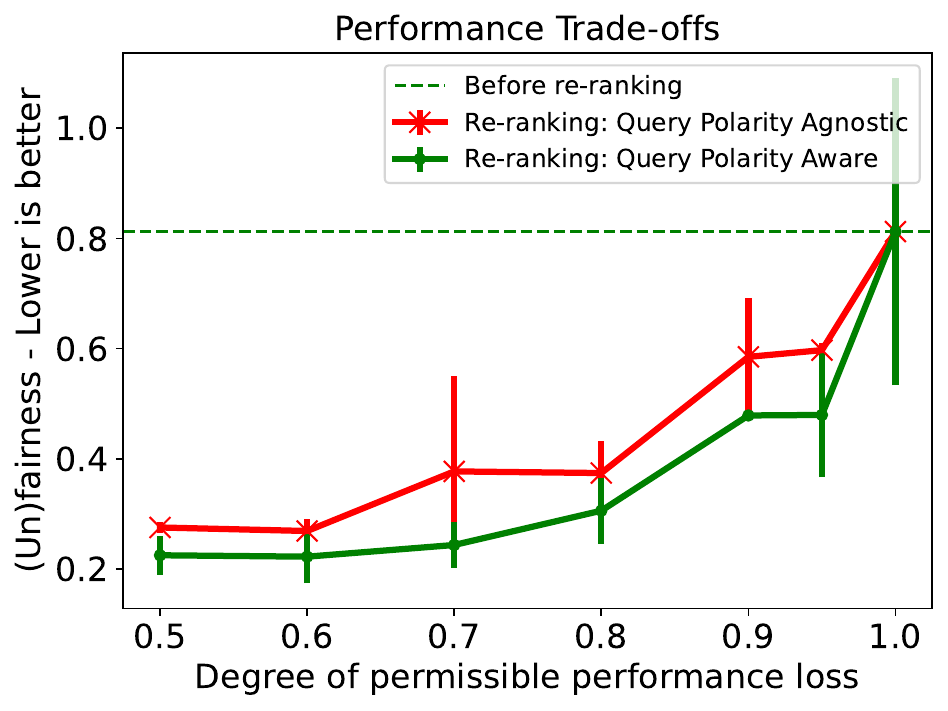}
    \caption{Fairness over different choices of $\theta$ for \texttt{rateMDs}.}
\label{fig:tradeoffs_ratemd}

\end{figure}

\section{Online vs Offline Optimization}
We observe that online optimization underperforms or performs similarly fully offline fairness optimization on three of the four datasets (see Table~\ref{fig:online_offline_optimization}).

\section{Multiple properties per query}
\label{sec:multiple_properties}

\begin{table}
\caption{Impact of fair ranking with a vector of query properties}\label{tab:vector}
\adjustbox{max width=\linewidth}{
\begin{tabular}{llrr}
\toprule
 &  & Pre-intervention & Post-intervention \\
Dataset & Measure &  &  \\
\midrule
\multirow[t]{4}{*}{\texttt{synth-binary}} & IAA & 27.75 & 13.78\\
 & DistFaiR($L_1$) & 2.21 & 0.60 \\
\cline{1-4}
\multirow[t]{4}{*}{\texttt{synth-cont}} & IAA & 17.82 & 13.26 \\
 & DistFaiR($L_1$) & 0.92 & 0.59 \\
\bottomrule
\end{tabular}}
\end{table}
We empirically conduct experiments where each query in the synthetic dataset contains three total properties. We define fairness as the sum of fairness metrics with each component considered separately.

We observe that online optimization reduces unfairness from across metrics in the \texttt{synth-binary} and \texttt{synth-cont} datasets in Table~\ref{tab:vector}.

\end{document}